\documentclass[11pt]{article}
\usepackage[utf8]{inputenc}
\usepackage[margin=1in]{geometry}

\usepackage{microtype}
\usepackage{graphicx}
\usepackage{subfigure,wrapfig}
\usepackage{booktabs}
\usepackage{natbib}
\setcitestyle{open={(},close={)}}

\usepackage[colorlinks=true,citecolor=blue]{hyperref}

\usepackage{booktabs}       % professional-quality tables
\usepackage{amsfonts}       % blackboard math symbols
\usepackage{amsmath}
\usepackage{graphicx}
\usepackage{amsthm}
\usepackage{amssymb}
\usepackage{multirow}
\usepackage{boldline}
\usepackage{makecell}
\usepackage{algpseudocode}
\usepackage{algorithm}
\usepackage{amssymb}
\usepackage{mathtools}
\usepackage{stmaryrd}
\usepackage[T1]{fontenc}
\usepackage{xargs}
\usepackage{xcolor}
\usepackage{wrapfig}
\usepackage{bm,bbm}

\newcommand{\appendixhead}%
{\hspace{2.3in}\textbf{\huge Appendices}
\vspace{0.2in}}

\usepackage{xargs}
\usepackage{stmaryrd}
\usepackage{microtype}
\usepackage{graphicx}
\usepackage{booktabs} % for professional tables
\usepackage{mdframed}
\usepackage{multirow}

\theoremstyle{plain}
\newmdtheoremenv{theo}{Theorem}[section]
\newmdtheoremenv{coro}[theo]{Corollary}
\newmdtheoremenv{lem}[theo]{Lemma}

\theoremstyle{definition}

\newtheorem{assumption}[theo]{Assumption}
\theoremstyle{remark}

\def\E{\mathbb{E}}

\newtheorem*{Lemma*}{Lemma}
\newtheorem*{Theorem*}{Theorem}
\newtheorem*{Corollary*}{Corollary}
 
\newcommand{\eqsp}{\;}
\newcommand{\beq}{\begin{equation}}
\newcommand{\eeq}{\end{equation}}
\newcommand{\eqdef}{\mathrel{\mathop:}=}
\def\EE{\mathbb{E}}
\newcommand{\norm}[1]{\left\Vert #1 \right\Vert}
\newcommand{\pscal}[2]{\left\langle#1\,|\,#2 \right\rangle}

\def\rset{\ensuremath{\mathbb{R}}}
\newcommand{\inter}{\llbracket n \rrbracket}

\def\tot{\mathsf{h}}

\newcommand{\sign}{\text{sign}}

\begin{document}

\title{\bf Layer-wise and Dimension-wise Locally Adaptive Federated Learning}

\author{\vspace{0.3in}\\\textbf{Belhal Karimi,\ Ping Li,\ Xiaoyun Li} \\\\
Cognitive Computing Lab\\
Baidu Research\\
10900 NE 8th St. Bellevue, WA 98004\\\\
  \texttt{\{belhal.karimi,\ pingli98,\ lixiaoyun996\}@gmail.com}
}
\date{\vspace{0.3in}}
\maketitle

\begin{abstract}
In the emerging paradigm of Federated Learning (FL), large amount of clients such as mobile devices are used to train possibly high-dimensional models on their respective data. Combining (\textit{dimension-wise}) adaptive gradient methods (e.g. Adam, AMSGrad)
with FL has been an active direction, which is shown to outperform traditional SGD based FL in many cases. In this paper, we focus on the problem of training federated deep neural networks, and propose a novel FL framework which further introduces \emph{layer-wise} adaptivity to the local model updates. Our framework can be applied to locally adaptive FL methods including two recent algorithms, Mime~\citep{karimireddy2020mime} and Fed-AMS~\citep{chen2020toward}.
Theoretically, we provide a convergence analysis of our layer-wise FL methods, coined Fed-LAMB and Mime-LAMB, which matches the convergence rate of state-of-the-art results in FL and exhibits linear speedup in terms of the number of workers. Experimental results on various datasets and models, under both IID and non-IID local data settings, show that both Fed-LAMB and Mime-LAMB achieve faster convergence speed and better generalization performance, compared to the various recent adaptive FL methods.
\end{abstract}

\newpage

\section{Introduction}\label{sec:introduction}

A growing and important task while learning models on observed data, is the ability to train over a large number of clients which could either be personal devices or distinct entities.
In the paradigm of Federated Learning (FL)~\citep{konevcny2016federated,mcmahan2017communication}, a central server orchestrates the optimization over those clients under the constraint that the data can neither be gathered nor shared among the clients.
This is computationally more efficient, since more distributed computing resources are used; also, this is a very practical scenario which allows individual data holders (e.g., mobile devices) to train a  model jointly without leaking private data.
In this paper, we consider the following optimization problem:
\begin{equation}\label{eq:opt}
\min_{\theta} f(\theta) \eqdef \frac{1}{n} \sum_{i=1}^n f_i(\theta)= \frac{1}{n} \sum_{i=1}^n \mathbb E_{\xi\sim \mathcal X_i}[F_i(\theta;\xi)],
\end{equation}
where the nonconvex function (e.g., deep networks) $f_i$ represents the average loss over the local data samples for worker $i \in \inter$, and $\theta \in \mathbb R^d$ the global model parameter.
$\mathcal X_i$ is the data distribution on each client $i$.
While \eqref{eq:opt} reminds that of standard distributed optimization, the principle and setting of FL are different from the classical distributed paradigm: (i) Local updates: FL allows clients to perform multiple updates on the local models before the global aggregation; (ii) Data heterogeneity: in FL, the local data distributions $\mathcal X_i$ are usually different across workers, hindering the convergence of the global model.
FL aims at finding a solution of \eqref{eq:opt} in fewest number of communication rounds.

\vspace{0.1in}

One of the most popular framework for FL is called Fed-SGD~\citep{mcmahan2017communication}: we adopt multiple local Stochastic Gradient Descent (SGD) steps in each device, send those local models to the server that computes the average over the received local model parameters, and broadcasts it back to the devices. Moreover, momentum can be added to local SGD training for faster convergence and better learning performance~\citep{Proc:YuJY_ICML19}. On the other hand, adaptive gradient methods, e.g., Adam~\citep{KB15}, AMSGrad~\citep{reddi2019convergence}, have shown great success in many deep learning tasks. For instance, the update rule of Adam reads as
\begin{equation}\label{rule:adam}
    \begin{aligned}
       \theta_t=\theta_{t-1}-\frac{\alpha m_t}{\sqrt{v_t}},\ \quad &m_t=\beta_1 m_{t-1}+(1-\beta_1) g_t,\\
       & v_t=\beta_2 v_{t-1}+(1-\beta_2) g_t^2,
    \end{aligned}
\end{equation}
where $\alpha$ is the learning rate and $g_t$ is the gradient at time $t$.
We note that the effective learning rate of Adam is $\alpha/\sqrt{v}$, which is different across dimensions, i.e., \textit{dimension-wise} adaptive. Recently, we have seen growing research efforts in the design of FL frameworks that adopt adaptive gradient methods as the protocols for local model training instead of SGD.
Examples include federated AMSGrad (Fed-AMS)~\citep{chen2020toward} and Mime~\citep{karimireddy2020mime} with Adam updates.
Specifically, in both methods, in each round the global server not only aggregates the local models, but also broadcasts to the workers a ``global'' second moment estimation to reconcile the dimension-wise adaptive learning rates across the clients. Therefore, this step can be regarded as a natural mitigation to data heterogeneity, which is a common and important practical scenario that affects the performance of FL algorithms~\citep{li2019federated,liang2019variance,karimireddy2019scaffold}.

\vspace{0.1in}

In this paper, we focus on improving adaptive FL algorithms. For (single-machine) training of deep neural networks using Adam,~\citet{you2019large} proposed a \textit{layer-wise} adjusted learning rate scheme called LAMB, where in each update, the ratio $m_t/\sqrt{v_t}$ is further normalized by the weight of the deep network, respectively for each layer.
LAMB allows large-batch training which could in particular speed up training large datasets and models like ImageNet~\citep{deng2009imagenet} and BERT~\citep{bert19}. Inspired by the acceleration effect of LAMB, we propose an improved framework for locally adaptive FL algorithms, integrating both \emph{dimension-wise} and \emph{layer-wise} adaptive learning rates in each device's local update. More specifically, \textbf{our contributions} are summarized as follows:

\begin{itemize}
\item We develop Fed-LAMB and Mime-LAMB, two instances of our layer-wise adaptive optimization framework for federated learning, following a principled layer-wise adaptive strategy to accelerate the training of deep neural networks.

\item We show that our algorithm converges at the rate of $\mathcal{O}\left(\frac{1}{\sqrt{n\tot R}} \right)$ to a stationary point, where $\tot$ is the number of layers of the network, $n$ is the number of clients and $R$ is the number of communication rounds. This matches the convergence rate of LAMB, AMSGrad, as well as the state-of-the-art results in federated learning. The theoretical communication efficiency matches that of Fed-AMS~\citep{chen2020toward}.

\item We empirically compare several recent adaptive FL methods under both homogeneous and heterogeneous data setting on various benchmark datasets.
Our results confirm the accelerated empirical convergence of Fed-LAMB and Mime-LAMB over the baseline methods, including Fed-AMS and Mime.
In addition, Fed-LAMB and Mine-LAMB can also reach similar, or better, test accuracy than their corresponding baselines.
\end{itemize}

\section{Background and Related Work}\label{sec:related}

Firstly, we summarize some relevant work on adaptive optimization, layer-wise adaptivity and federated learning, which compose the key ingredients of our algorithm.

\vspace{0.1in}\noindent\textbf{Adaptive gradient methods.}
Adaptive methods have proven to be the spearhead for many nonconvex optimization tasks.
Gradient based optimization algorithms alleviate the possibly high nonconvexity of the objective function by adaptively updating each coordinate of their learning rate using past gradients.
Common used examples include RMSprop~\citep{TH12}, Adadelta~\citep{Z12}, Adam~\citep{KB15}, Nadam~\citep{dozat2016incorporating} and AMSGrad~\citep{reddi2019convergence}.
Their popularity owes to their great performance in training deep neural networks.
They generally combine the idea of adaptivity from AdaGrad~\citep{DHS11,mcmahan2010adaptive}, as explained above, and the idea of momentum from Nesterov's Method~\citep{N04} or Heavy ball method~\citep{P64} using past gradients.
AdaGrad displays superiority when the gradient is sparse compared to other classical methods~\citep{DHS11}. Yet, when applying AdaGrad to train deep neural networks, it is observed that the learning rate might decay too fast. Consequently,~\citep{KB15} developed Adam whose updating rule is presented in \eqref{rule:adam}.
A variant, called AMSGrad described in~\citet{reddi2019convergence}, forces $v$ to be monotone to fix the convergence issue. The convergence and generalization of adaptive methods are studied in, e.g.,~\citep{zhou2018convergence,Proc:Chen_ICLR19,zhou2020towards}. \cite{Proc:Li_ICLR22} extended adaptive gradient method to distributed setting with communication-efficient gradient compression.

\vspace{0.1in}\noindent\textbf{Layer-wise Adaptivity.} When training deep networks, in many cases the scale of gradients differs a lot across the network layers. When we use the same learning rate for the whole network, the update might be too preservative for some specific layers (with large weights) which may slow down the convergence. Based on this observation,~\citet{Proc:LARS18} proposed LARS, an extension of SGD with layer-wise adjusted scaling, whose performance, however, is not consistent across tasks. Later,~\citet{you2019large} proposed LAMB, an analogue layer-wise adaptive variant of Adam. The update rule of LAMB for the $\ell$-th layer of the network can be expressed as
\begin{align*}
    \theta_t^\ell=\theta_{t-1}^\ell-\frac{\alpha \| \theta_{t-1}^\ell\|}{\|\psi_t^\ell\|}\psi_t^\ell,\ \text{with}\ \psi_t^\ell=m_t^\ell/\sqrt{v_t^\ell},
\end{align*}
where $m_t$ and $v_t$ are defined in \eqref{rule:adam}. Intuitively, for the $\ell$-th layer, when the gradient magnitude is too small compared to the scale of the model parameter, we increase the effective learning rate to make the model move sufficiently far. Theoretically,~\citep{you2019large} showed that LAMB achieves the same convergence rate as Adam; empirically, LAMB can significantly accelerate the convergence of Adam, allowing the use of large mini-batch size with fewer training iterations for large datasets.

\vspace{0.1in}\noindent\textbf{Federated learning.}
An extension of the well known parameter server framework, where a model is being trained on several servers in a distributed manner, is called federated learning (FL)~\citep{konevcny2016federated,mcmahan2017communication} which has seen many applications in various fields~\citep{Article:YangLCT19,Proc:Leroy_ICASSP19,bonawitz2019towards,Article:NiknamDR20,Proc:XuGSWBW21}. For Fed-SGD (where clients perform SGD-based updates), recent variants and theoretical analysis on the convergence can be found in~\citet{Proc:YuJY_ICML19,karimireddy2019scaffold,Proc:Khaled_AISTATS20,Proc:Li_ICLR20,Proc:Woodworth_ICML20,Proc:WangTBR_ICLR20}.

Recently, several works have considered integrating adaptive gradient methods with FL. \citet{reddi2020adaptive} proposed Adp-Fed where the central server applies Adam-type updates. However, the local clients still perform SGD updates. \citet{chen2020toward,karimireddy2020mime} proposed Fed-AMS and Mime respectively, to adopt Adam/AMSGrad at the client level. Both works mitigate the influence of data heterogeneity by ``sharing'' the second moment $v$ (which controls the effective learning rate): in Fed-AMS, a global $v$ is computed and synchronized in each round by averaging the $v_i$'s, $i=1,...,n$ from the local clients; in Mime, a global $v$ is directly calculated using full-batches (averaged over all clients) and maintained at the central server. Hence, Mime requires at least twice computation as Fed-AMS. On many tasks, these methods outperform Fed-SGD and other popular methods like SCAFFOLD~\citep{karimireddy2019scaffold} and FedProx~\citep{Article:Sahu_arxiv18}.

\section{Layer-wise Adaptive Federated Learning}\label{sec:main}

\begin{figure}
\centering
    \includegraphics[width=4.5in]{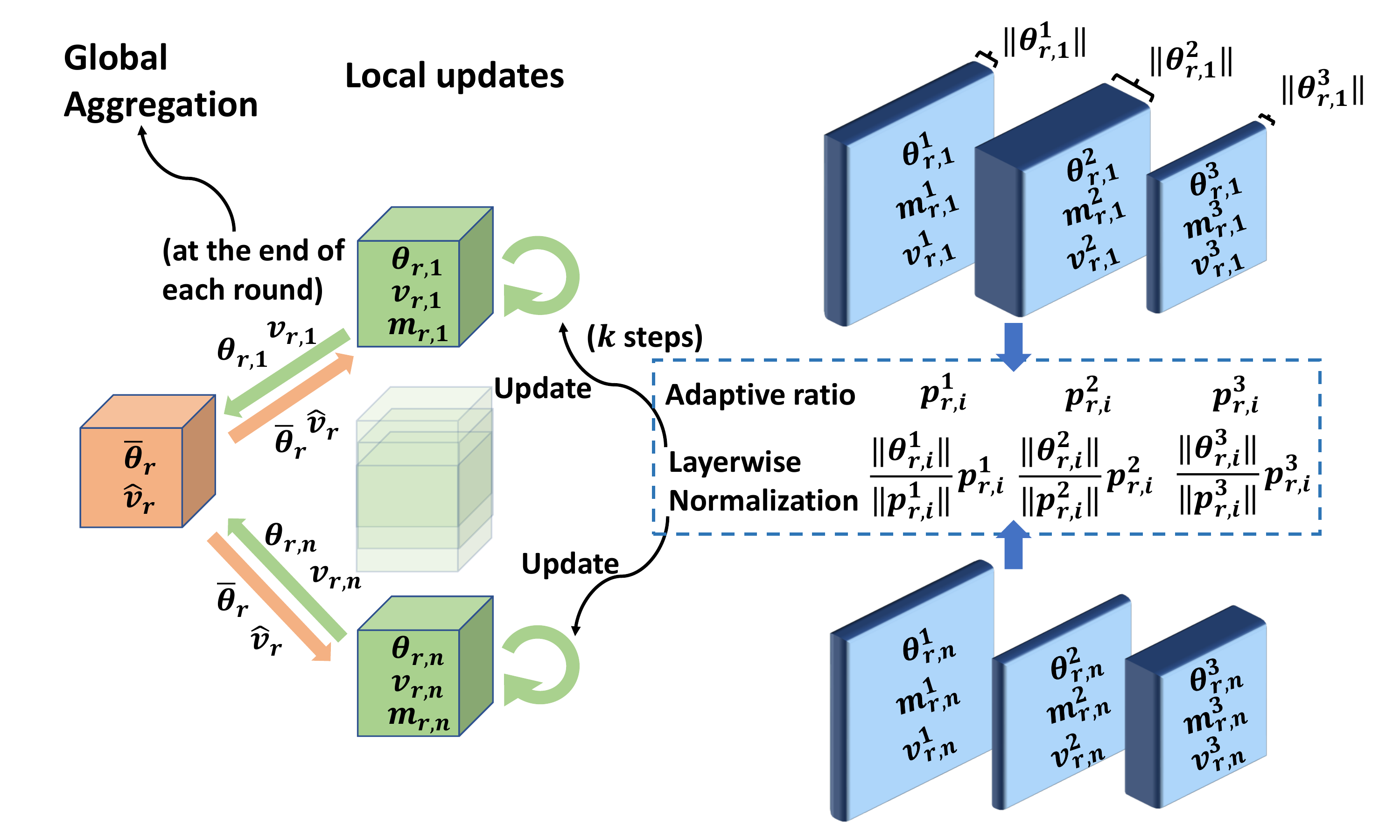}

	\caption{Illustration of Fed-LAMB framework (Algorithm~\ref{alg:ldams}), with a three-layer network and $\phi(x)=x$ as an example. The depth of each network layer represents the norm of its weights. For device $i$ and each local iteration in round $r$, the adaptive ratio of $j$-th layer $\psi_{r,i}^j$ is normalized according to $\Vert \theta_{r,i}^j\Vert$, and then used for updating the local model. At the end of each round $r$, client $i$ sends $\theta_{r,i} =  [\theta_{r,i}^{\ell}]_{\ell =1}^{\tot}$ and $v_{r,i}$ to the central server, which transmits back aggregated $\theta$ and $\hat v$ to devices to complete a round of training.}
	\label{fig:illustrate}
\end{figure}

In this section, we introduce our proposed federated learning framework, admitting both \textit{dimension-wise} adaptivity (of adaptive learning rate) and \textit{layer-wise} adaptivity (of layer-wise scaling). For conciseness, we mainly consider AMSGrad~\citep{reddi2019convergence} as the prototype method as it enjoys better theoretical convergence properties. We assume the loss function $f(\cdot)$ is induced by a multi-layer structured neural network, which includes a broad class of network architectures such as MLP, CNN, ResNet and Transformers.

\vspace{0.2in}
\noindent\textbf{Notations.} We denote by $\theta$ the vector of parameters taking values in $\rset^p$.
Suppose the neural network has $\tot$ layers, each with size $p_\ell$ (thus, $p= \sum_{\ell=1}^\tot p_\ell$). For each layer $\ell \in \llbracket \tot \rrbracket$, denote $\theta^\ell$ as the sub-vector corresponding to the $\ell$-th layer. Let $R$ be the number of communication rounds and $T$ be the number of local iterations per round. $\theta_{r,i}^{\ell,t}$ is the model parameter of layer $\ell$ at round $r$, local iteration $t$ and for worker $i$.

\begin{algorithm}[t]
\caption{ \colorbox{blue!20!white}{Fed-LAMB} and \colorbox{red!20!white}{Mime-LAMB} } \label{alg:ldams}
\begin{algorithmic}[1]

\State \textbf{Input}: parameter $0< \beta_1, \beta_2 <1$; learning rate $\alpha$; weight decaying rate $\lambda \in [0,1]$.
\State \textbf{Initialize}: $\theta_{0,i} \in \Theta \subseteq \mathbb R^d $; $m^0_{0,i}=\hat v^0_{0,i}=v^0_{0,i} = 0$, $\forall i\in \llbracket n\rrbracket$; $\bar{\theta}_0 =  \frac{1}{n} \sum_{i=1}^n \theta_{0,i}$; $\hat v_0=\epsilon$

\For{$r=1$ to $R$}
\State Sample a set of clients $D^r$
\For{parallel for device $i \in D^{r}$}
\State Set $\theta_{r,i}^{0} = \bar{\theta}_{r-1}$,\quad $m^{0}_{r,i} = m^T_{r-1,i}$\ ,\quad $v^{0}_{r,i} = \hat{v}_{r-1}$
\For{$t=1$ to $T$}
\State Sample a mini-batch from the local data
\State Compute stochastic gradient $g^t_{r,i}$ at $\theta_{r,i}^{t-1}$
\State $m^t_{r,i} = \beta_1 m^{t-1}_{r,i} + (1 - \beta_1) g^t_{r,i}$

\State \colorbox{blue!20!white}{$v^{t}_{r,i} = \beta_2 v^{t}_{r-1,i} + (1 - \beta_2) (g^t_{r,i})^2$ }

\State Compute the ratio  $\psi_{r,i}^t=m^{t}_{r,i}/(\sqrt{\hat v_{r-1}})$. \label{line:scale}

\State \label{line:layer} Update local model for each layer $\ell \in \llbracket \tot \rrbracket$:
\begin{equation}\label{eq:updatelayer}
    \theta_{r,i}^{\ell,t}=\theta_{r,i}^{\ell,t-1}-\frac{\alpha_{r}\phi(\|\theta_{r,i}^{\ell,t-1}\|)(\psi_{r,i}^{\ell,t}+\lambda \theta_{r,i}^{\ell,t-1})}{\|\psi_{r,i}^{\ell,t}+\lambda \theta_{r,i}^{\ell,t-1}\|}
\end{equation}
\EndFor
\State Communicate $\theta_{r,i}^{T} = [\theta_{r,i}^{\ell,T}]_{\ell =1}^{\tot}$ to server

\State \colorbox{blue!20!white}{Communicate $v_{r,i}^T$ to server}

\State \colorbox{red!20!white}{Communicate $\nabla f_i(\bar\theta_{r-1})$ using full local data}

\EndFor

\State Server compute $\bar{\theta}_r = \frac{1}{|D^{r}|} \sum_{i \in D^{r}} \theta_{r,i}^{T}$

\State \colorbox{blue!20!white}{Server compute $\hat{v}_{r} = \max( \hat{v}_{r-1},\frac{1}{|D^{r}|} \sum_{i \in D^{r}} v^T_{r,i} )$}

\State \colorbox{red!20!white}{Compute $\nabla f(\bar \theta_{r-1})=\frac{1}{|D^r|}\sum_{i\in D_r}\nabla f_i(\bar \theta_{r-1})$}

\State \colorbox{red!20!white}{Compute $v_r = \beta_2 v_{r-1}+(1-\beta_2)\nabla f(\bar \theta_{r-1})^2)$}

\State \colorbox{red!20!white}{Update $\hat v_{r}=\max(\hat v_{r-1},v_r$)}

\EndFor
\end{algorithmic}

\end{algorithm}

\newpage

In general, our proposed algorithm can be viewed as a novel extension of LAMB to the more complicated federated setting.
Based on the two recent works regarding locally adaptive FL mentioned above, we present the framework with two instances, Fed-LAMB and Mime-LAMB, as summarized in Algorithm~\ref{alg:ldams} and depicted in Figure~\ref{fig:illustrate}. We differentiate the steps of these two methods by blue\colorbox{blue!20!white}{(Fed-LAMB)}and red\colorbox{red!20!white}{(Mime-LAMB)}boxes surrounding the text. Both methods use layer-wise adaptive LAMB for local updates (Line~13). The update in \eqref{eq:updatelayer} on local workers can be expressed as
\begin{align*}
    \theta \leftarrow \theta-\alpha\frac{\phi(\|\theta\|)}{\|\psi+\lambda\theta\|}(\psi+\lambda\theta),
\end{align*}
where $\phi(\cdot): \mathbb R_+ \mapsto \mathbb R_+$ is a scaling function (usually chosen to be the identity function in practice) and $\lambda$ is the weight decay rate. The main difference between Fed-LAMB and Mime-LAMB is the way the second moment $\hat v$ is synchronized, i.e., the dimension-wise adaptive learning rate.
Both methods maintain a global $\hat v$ at the central server:
\begin{itemize}
    \item \colorbox{blue!20!white}{Fed-LAMB (Line 20)}: at the end of each round, clients $i$ communicates the local $v_{i}$; the server updates the global $\hat v$ by max operation with the averaged $v$, and sends back the $\hat v$.

    \item \colorbox{red!20!white}{Mime-LAMB (Line 21-23)}: in each round $r$, the client computes and transmits the gradient at the global model $\bar\theta_r$ using full local data; the server updates the global $v$ and $\hat v$ in the same manner as AMSGrad.
\end{itemize}
Conceptually, both approaches aim at alleviating the impact of data heterogeneity by ``globally'' reconciling the adaptive learning rates. Note that Mime-LAMB needs to calculate the gradients twice, leading to double the computational cost as Fed-LAMB.

\newpage

\noindent\textbf{Data Heterogeneity:} Dealing with data heterogeneity is an important topic in federated learning. Works have been conducted (e.g.,~\citet{karimireddy2019scaffold}) on designing specific strategies to alleviate the negative influence of non-IID data by using techniques like control variate. We note that our scheme is, in some sense, naturally capable of balancing the heterogeneity in different local data distributions. This is largely due to the ``moment sharing'' steps in Algorithm~\ref{alg:ldams}), where the adaptive learning rates guided by the second moment estimation are aggregated among clients periodically. In~\citet{chen2020toward} and~\citet{karimireddy2020mime}, the authors have shown that Fed-AMS and Mime would perform much worse, or even diverge, without aggregating and sharing the second moment $\hat v$ (please refer to the papers for details). Intuitively, synchronizing $\hat v$ makes all the clients “on the same pace” which is crucial for the convergence of locally adaptive FL methods. We will provide more discussion on this moment synchronization in the next section.

\section{Theoretical Analysis}\label{sec:theory}

For conciseness, we summarize in Table~\ref{tab:notationsapp} some important notations that will be used in our analysis.

\begin{table}[h]

\begin{center}

\begin{tabular}{r c p{12cm} }
\toprule
$R, T$ & $\eqdef$ &  Number of communications rounds and local iterations (resp.)\\
$n, D, i$ & $\eqdef$ &  Total number of clients, portion sampled uniformly and client index \\
$\tot, \ell$ & $\eqdef$ &  Total number of layers in the DNN and its index \\
$\phi(\cdot)$ & $\eqdef$ &  Scaling factor in Fed-LAMB update\\
$\bar{\theta}$ & $\eqdef$ &  Global model (after periodic averaging)\\
$\psi_{r,i}^{t}$ & $\eqdef$ &  ratio computed at round $r$, local iteration $t$ and for device $i$. $\psi_{r,i}^{\ell,t}$ denotes its component at layer $\ell$\\
\bottomrule
\end{tabular}
\end{center}
\caption{Summary of notations used in the paper.}
\label{tab:notationsapp}
\end{table}

We need the following analytical assumptions.

\begin{assumption}\label{ass:smooth}(Smoothness)
For all $i \in \inter$ and $\ell \in \llbracket \tot \rrbracket$, the local loss function is $L_\ell$-smooth: $\norm{\nabla f_i (\theta^\ell) - \nabla f_i (\vartheta^\ell)} \leq L_\ell \norm{\theta^\ell-\vartheta^\ell}$.
\end{assumption}
\begin{assumption}\label{ass:boundgrad}(Unbiased and bounded gradient)
The stochastic gradient is unbiased for $\forall r,t,i$: $\EE[g_{r,i}^t] = \nabla f_i(\theta_r^t)$ and bounded by $\norm{g_{r,i}^t} \leq M$.
\end{assumption}
\begin{assumption}\label{ass:var}(Bounded variance) The stochastic gradient admits (\emph{locally}) $\EE[|g_{r,i}^j - \nabla f_i(\theta_r)^j|^2] < \sigma^2$, and (\emph{globally}) $ \frac{1}{n} \sum_{i=1}^n ||\nabla f_{i}(\theta_r) - \nabla f(\theta_r)||^2] < G^2$.
\end{assumption}
Assumption~\ref{ass:smooth} and Assumption~\ref{ass:boundgrad} are commonly used in the analysis of adaptive gradients methods~\citep{reddi2019convergence,Proc:Chen_ICLR19,karimireddy2020mime,reddi2020adaptive}, Assumption~\ref{ass:var} characterizes the data heterogeneity among local devices, and $G = 0$ when local data are IID Following~\citep{you2019large}, we use the following assumption on the scaling function $\phi$.
\begin{assumption}\label{ass:phi}
For any $a \in \rset_+$, there exist $\phi_m>0,\phi_M>0$ such that $\phi_m \leq  \phi(a) \leq \phi_M$.
\end{assumption}
We now state our main result regarding the convergence rate of the proposed Algorithm~\ref{alg:ldams}.

\begin{theo}\label{th:multiple update}
Under \textbf{Assumption~\ref{ass:smooth}-Assumption~\ref{ass:phi}}, consider $\{\overline{\theta_r}\}_{r>0}$ obtained from Algorithm~\ref{alg:ldams} with a constant learning rate $\alpha$. Let $\lambda = 0$. Then, for any round $R > 0$,~we~have
\begin{align}
\frac{1}{R}\sum_{r=1}^R  \EE\left[ \left\| \frac{\nabla f(\overline{\theta_r})}{\hat v_r^{1/4}}   \right \|^2 \right] &\leq    \sqrt{\frac{M^2 p}{n}}  \frac{ \triangle}{\tot \alpha R}+\frac{4 \alpha^2 \overline{L} M^2 (T-1)^2 \phi_M^2 (1-\beta_2)p}{\sqrt{\epsilon}} \notag \\
&+4\alpha \frac{M^2}{\sqrt{\epsilon}} +      \frac{\phi_M   \sigma^2}{R n} \sqrt{\frac{1 - \beta_2}{M^2 p}  } + 4\alpha \left[ \phi_M^2\sqrt{M^2+p\sigma^2} \right]     \notag\\
& +4  \frac{\alpha^2 \overline{L}}{\sqrt{\epsilon}}  M^2 (T-1)^2 G^2 (1-\beta_2)p +4\alpha \left[\phi_M \frac{\tot \sigma^2}{\sqrt{n}}\right], \label{bound1multiple}
\end{align}
where $\triangle=\EE[f(\bar{\theta}_1)]  - \min \limits_{\theta \in \Theta} f(\theta)$ and $\overline{L}=\sum_{\ell=1}^\tot L_\ell$.
\end{theo}

Note that this result holds for both Fed-LAMB and Mime-LAMB variants. Also, the manifestation of $p$ in the rate is because the variance bound is assumed on each dimension in Assumption~\ref{ass:var}. This dependency on $p$ can be removed when Assumption~\ref{ass:var} is assumed globally, which is also common in optimization literature.

Using a uniform bound on the moment $\|\hat v_r \| \leq M^2$ and by choosing a suitable decreasing learning rate, we have the following simplified statement.

\vspace{0.1in}
\begin{coro}\label{coro:main}
Under the same setting as Theorem~\ref{th:multiple update}, with $\alpha = \mathcal{O}(\frac{1}{ \sqrt{ \tot R}})$, it holds that
\begin{align} \label{coro:rate}
\frac{1}{R}\sum_{r=1}^R  \EE\left[ \left\| \nabla f(\overline{\theta_r})   \right \|^2 \right] \leq \mathcal{O}\left( \frac{\sqrt p}{\sqrt{n\tot R}}+\frac{\sqrt\tot \sigma^2 }{\sqrt{nR}}  + \frac{G^2(T-1)^2p}{R\tot}\right).
\end{align}
\end{coro}
\vspace{0.1in}

The leading two terms display a dependence of the convergence rate of Fed-LAMB on the initialization and the local variance of the stochastic gradients (Assumption~\ref{ass:var}). The last term involves the number of local updates $T$, and the global variance $G^2$ characterizing the data heterogeneity. Next, we provide detailed discussion and comparison of our bound to related prior results.

\vspace{0.1in}\noindent\textbf{LAMB bound in~\citet{you2019large}: }
We start our discussion with the comparison of convergence rate of Fed-LAMB with that of LAMB, Theorem 3 in~\citet{you2019large}. In a single-machine setting, the convergence rate of LAMB is $\mathcal O(\sqrt{p}{\sqrt{\tot T}})$ where $T$ is the number of training iterations. Note the convergence rate of Fed-LAMB is different from that of LAMB in the sense that, the convergence criterion is given at the averaged parameters (global model) at the end of each round. In Corollary~\ref{coro:main}, our rate would match LAMB if we take number of local step $T=1$. This also holds true for any fixed $T$ and $R$ sufficiently large. In addition, the $\mathcal O(\frac{1}{\sqrt{nR}})$ rate of Fed-LAMB implies a \textit{linear speedup} effect: the number of iterations to reach a $\delta$-stationary point of Fed-LAMB decreases linearly in $n$, which displays the merit of distributed (federated) learning.

\vspace{0.1in}\noindent\textbf{Fed-AMS bound in~\citet{chen2020toward}: }
We now compare our method theoretically with Fed-AMS, the baseline distributed adaptive method developed in~\citet{chen2020toward}. Their results state that when $T\leq \mathcal O(R^{1/3})$, the convergence rate of Fed-AMS is $\mathcal O(\sqrt{\frac{p}{nR}})$. Firstly, when the number of rounds $R$ is sufficiently large, both our rate~\eqref{coro:rate} and the rate of Fed-AMS are dominated by $\mathcal O(\frac{\sqrt p}{\sqrt{n R}})$, improving the convergence rate of the standard AMSGrad, e.g.~\citep{Arxiv:Zhou_18} by $\mathcal O(1/\sqrt n)$. Secondly, in~\eqref{coro:rate}, the last term containing the number of local updates $T$ is small as long as $T^4\leq \mathcal O(\frac{Rh}{G^2})$. If we further assume $h\simeq T$, then we get the same rate of convergence as Fed-AMS with $T\leq \mathcal{O}(R^{1/3})$ local iterations, identical to the condition of Fed-AMS. Moreover, when $G$ is small (less data heterogeneity), the bound on $T$ would increase, i.e., we can conduct more local updates. This is intuitive, for example, when $G=0$ in the IID data setting, $T$ can be very large.

In general, similar to the result that single machine LAMB matches the convergence rate of Adam theoretically, we show that Fed-LAMB achieves the same convergence rate as Fed-AMS, in the federated (distributed) setting, which is the first such result in literature for locally adaptive FL method. Also similar to the findings in~\citet{you2019large}, we will show that Fed-LAMB and its variants provide impressive acceleration empirically, in our experimental study (Section~\ref{sec:numerical}).

\newpage

\vspace{0.1in}\noindent\textbf{Practical Considerations:} Under our analytic setting and conditions mentioned above, the convergence rate of Fed-LAMB also matches many popular federated learning methods in nonconvex optimization, e.g., Fed-SGD~\citep{mcmahan2017communication}, Mime~\citep{karimireddy2020mime} and Adp-Fed~\citep{reddi2020adaptive}, at $\mathcal O(\frac{1}{\sqrt{nRT}})$. In practice, when trained with same number of rounds and $R$ and local iterations $T$, Mime, Fed-LAMB and Fed-Mime all require communicating two tensors, while Fed-SGD, Adp-Fed only communicate one tensor (the local model parameter). Next, we discuss a simple implementation trick of our algorithm that leads to less communication.

\vspace{0.1in}
\noindent\textbf{Extension: lazy synchronization of $\hat v_t$.} As mentioned before, the synchronization of the second moment $\hat v$ local workers keep in step with each other. Yet, may not need to update and broadcast the global $\hat v$ in every round. To reduce the extra communication overhead of transmitting $\hat v$, one trick in practice is to reduce the aggregation frequency of $\hat v$ (e.g., we synchronize $\hat v$ every $Z$ rounds). It can be shown that this ``lazy'' aggregation does not affect the convergence rate of our Fed-LAMB. Yet it can effectively reduce the communication of $\hat v$ by a factor of $Z$, which to a great extent alleviates the extra communication cost of locally adaptive methods. We will also show empirical evidence of this trick in our experiments.

\section{Experiments}\label{sec:numerical}

In this section, we conduct experiments on benchmark datasets with various network architectures to justify the effectiveness of our proposed method in practice. Our method empirically confirms its merit in terms of convergence speed. Basically, Fed-LAMB and Mime-LAMB reduce the number of rounds and thus the communication cost required to achieve a similar stationary point (or test accuracy) than the baseline methods.
In many cases, Fed-LAMB also brings notable improvement in generalization over baselines.

\vspace{0.1in}\noindent\textbf{Methods.} We evaluate the following five FL algorithms, mainly focusing on recent federated optimization approaches based on adaptive methods:
\begin{enumerate}
    \item Fed-SGD~\citep{mcmahan2017communication}, standard federated averaging with local SGD updates.

    \item Adp-Fed (\emph{Adaptive Federated Optimization}, see Appendix~\ref{app:experiment}), the federated adaptive algorithm proposed by~\citep{reddi2020adaptive}. Adp-Fed performs local SGD updates. In each round $r$, the changes in local models, $\triangle_i=w_{r,i}^T-w_{r,i}^0$, $i=1,...,n$, are sent to the central server for an aggregated Adam update.

    \item Fed-AMS~\citep{chen2020toward}, locally adaptive AMSGrad.

     \item Mime~\citep{karimireddy2020mime} with AMSGrad, which performs adaptive local updates with central-server-guided global adaptive learning rate.

    \item Our proposed Fed-LAMB and Mime-LAMB (Algorithm~\ref{alg:ldams}).
\end{enumerate}
For all the adaptive gradient methods, we set $\beta_1=0.9$,~$\beta_2=0.999$ by default~\citep{reddi2019convergence}. We present the results of $n=50$ clients with $0.5$ participation rate, i.e., we randomly pick half of the clients to be active for training in each round, and the local mini-batch size is set as 128. In each round, the training samples are allocated to the active devices, and one local epoch is completed after all the local devices run one pass over their received samples via mini-batch training. Example results with more clients can be found in Figure~\ref{fig:client200}.

We tune the initial learning rate $\alpha$ for each algorithm over a fine grid. For Adp-Fed, there are two learning rates involved (global and local), both of which are tuned. More tuning details can be found in Appendix~\ref{app:experiment}. For Fed-LAMB and Mime-LAMB, the weight decay rate $\lambda$ is tuned from $\{0,0.01,0.1\}$, and we use the identity scaling function for $\phi(\cdot)$. For each run, we report the best test accuracy. The results are averaged over 3 runs, each with same initialization for every method.

\vspace{0.1in}\noindent\textbf{Datasets and models.} We experiment with four popular benchmark image classification datasets: MNIST~\citep{lecun1998mnist}, Fashion MNIST (FMNIST)~\citep{xiao2017fashion}, CIFAR-10~\citep{krizhevsky2009learning} and TinyImageNet~\citep{deng2009imagenet}. For MNIST, we apply 1) a simple multi-layer perceptron (MLP), which has one hidden layer containing 200 cells; 2) Convolutional Neural Network (CNN), which has two max-pooled convolutional layers followed by a dropout layer and two fully-connected layers with 320 and 50 cells respectively. This CNN is also implemented for FMNIST.
For CIFAR-10 and TinyImageNet, we use ResNet-18 ~\citep{Proc:He-resnet16}.

\subsection{Comparison under IID settings}

\begin{figure}[h]

    \begin{center}
        \mbox{\hspace{-0.15in}
        \includegraphics[width=2.25in]{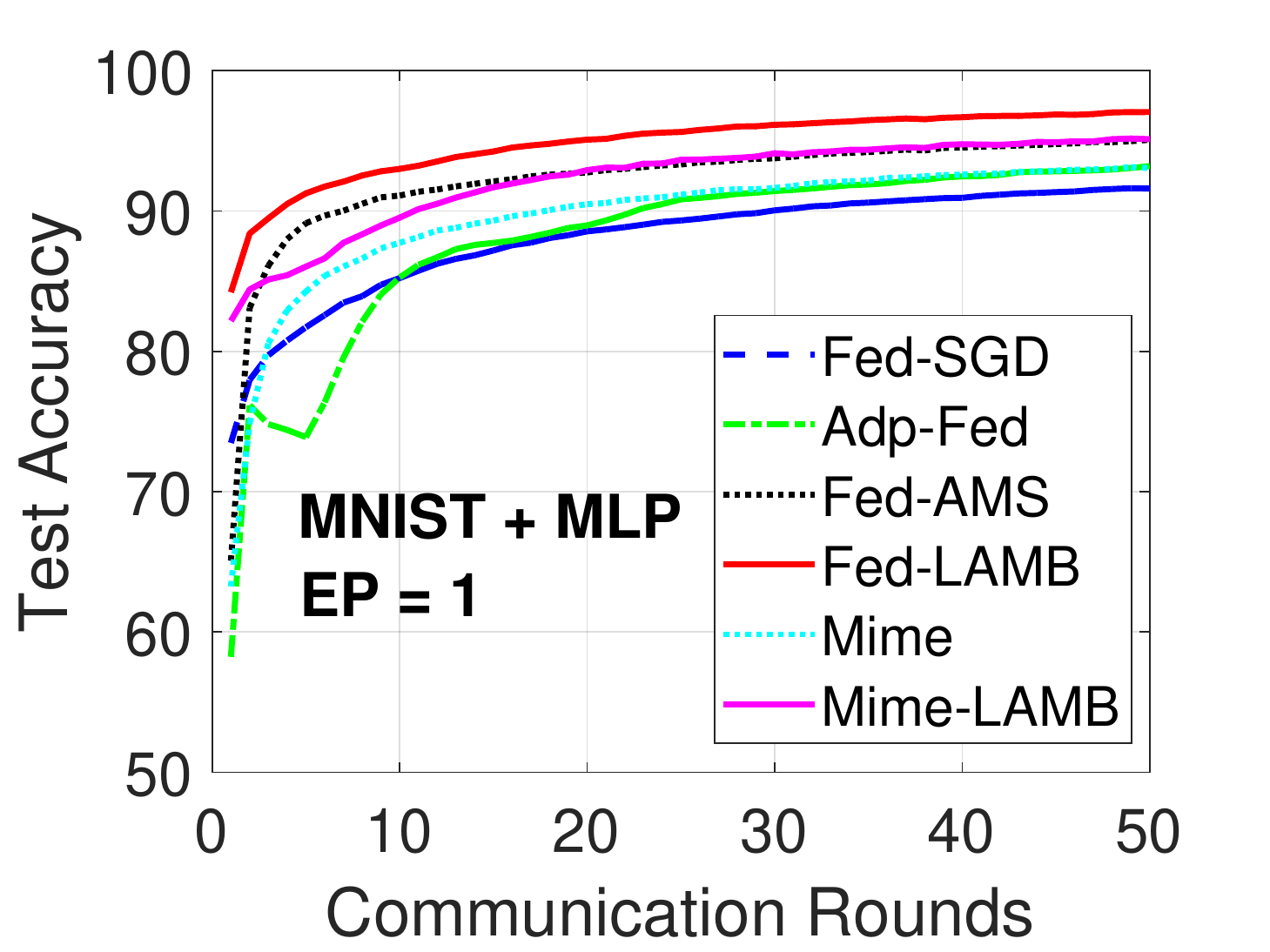}\hspace{-0.1in}
        \includegraphics[width=2.25in]{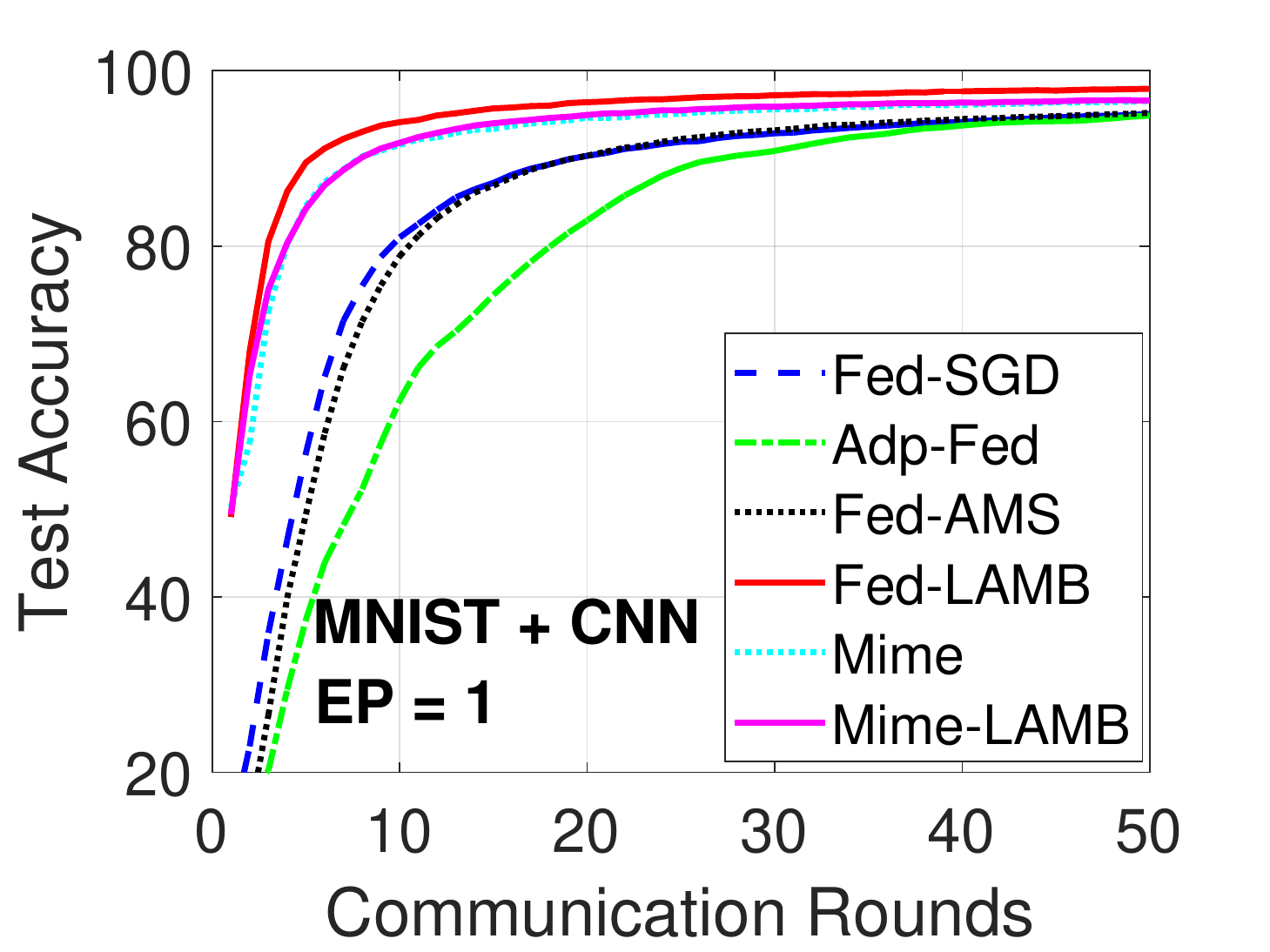}\hspace{-0.1in}
        \includegraphics[width=2.25in]{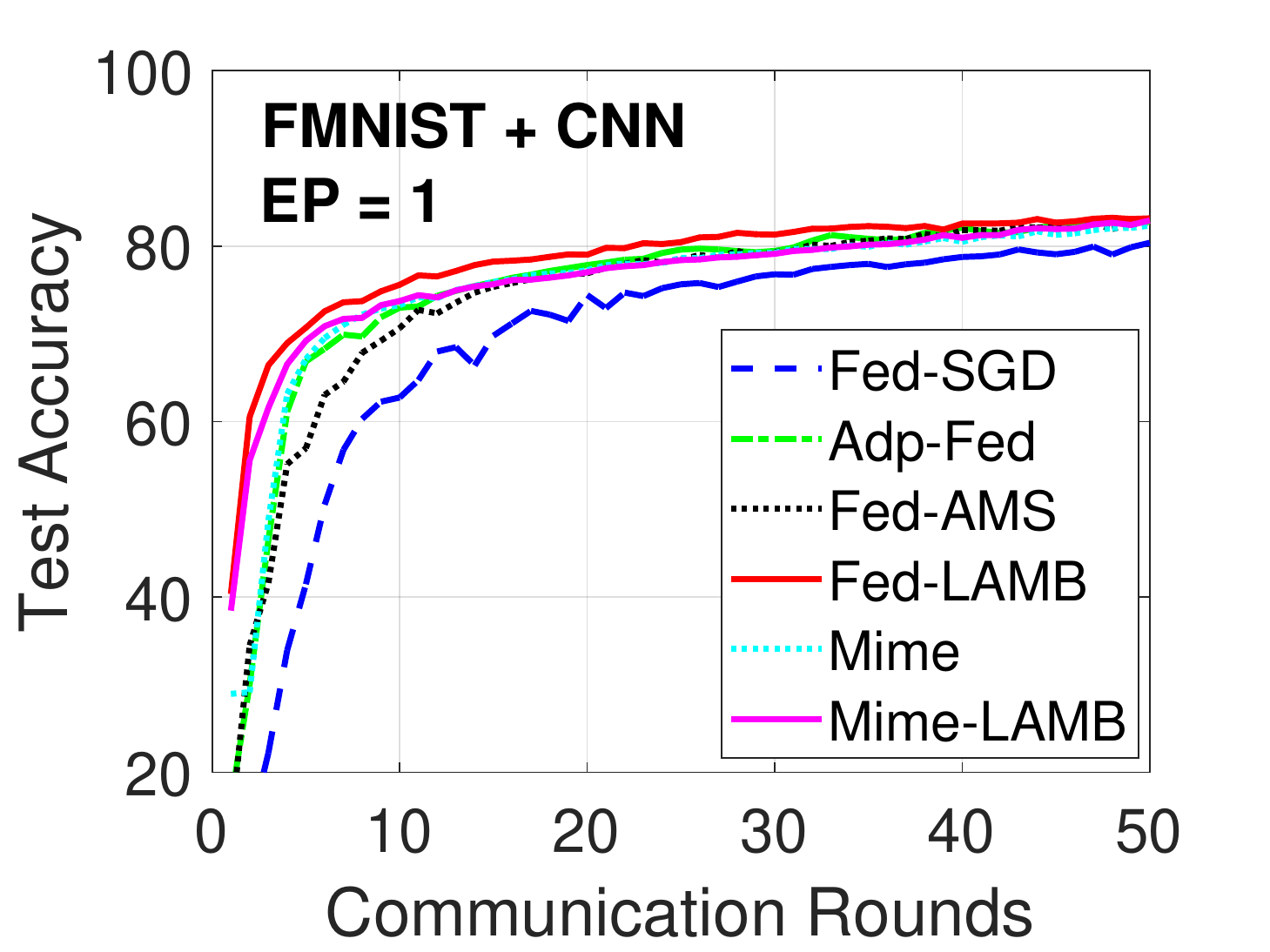}
        }
        \mbox{\hspace{-0.15in}
        \includegraphics[width=2.25in]{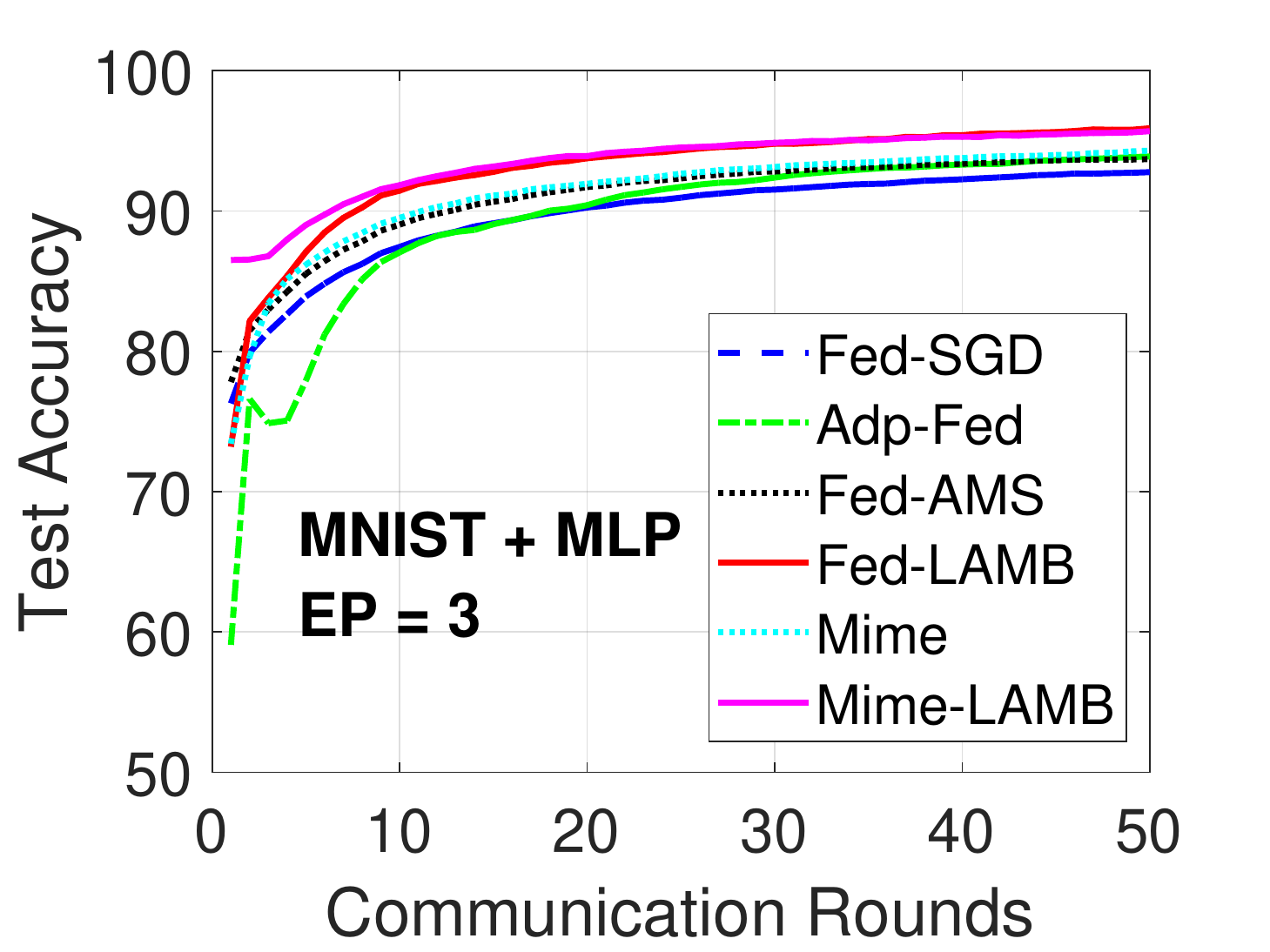}\hspace{-0.1in}
        \includegraphics[width=2.25in]{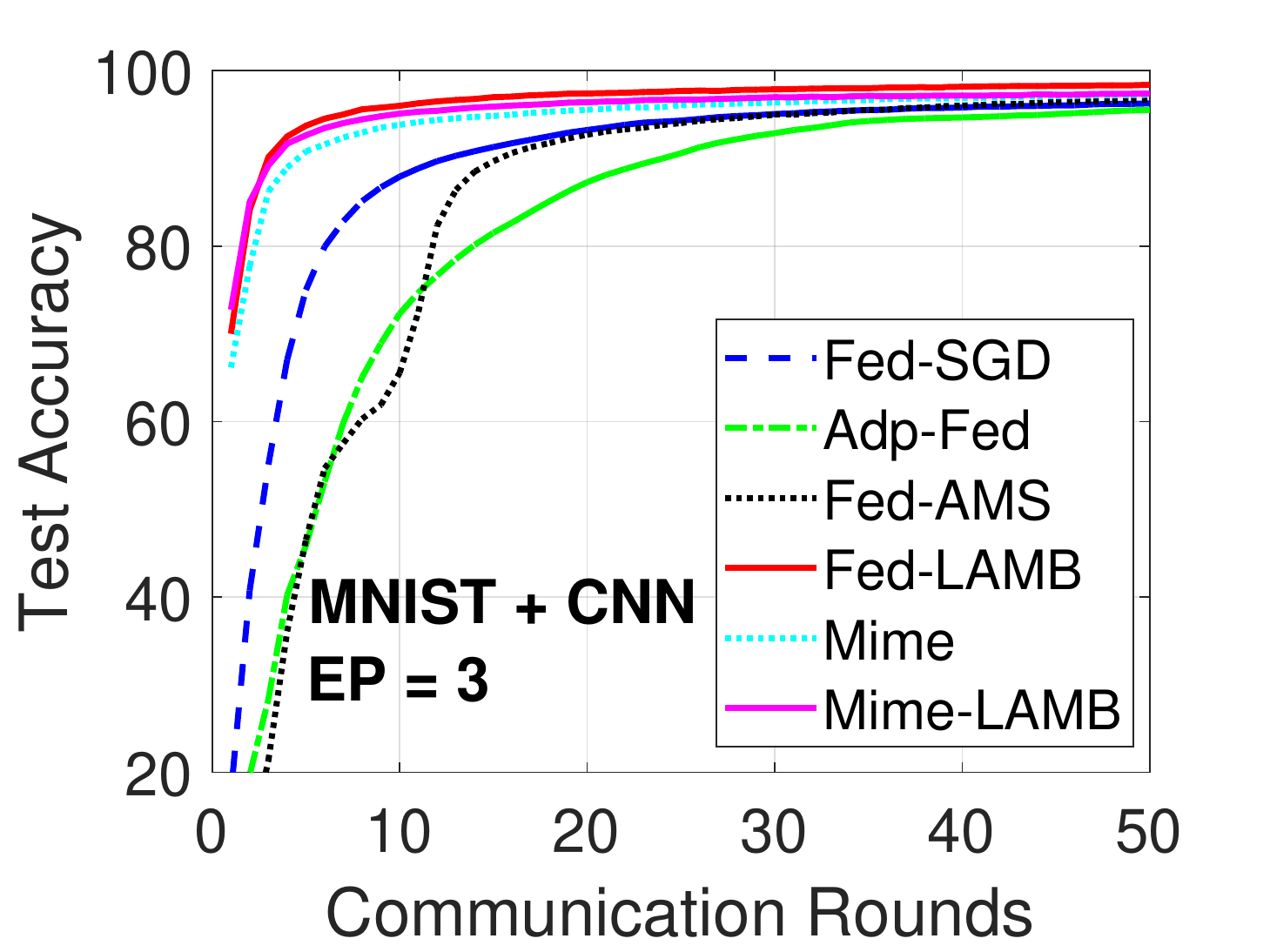}\hspace{-0.1in}
        \includegraphics[width=2.25in]{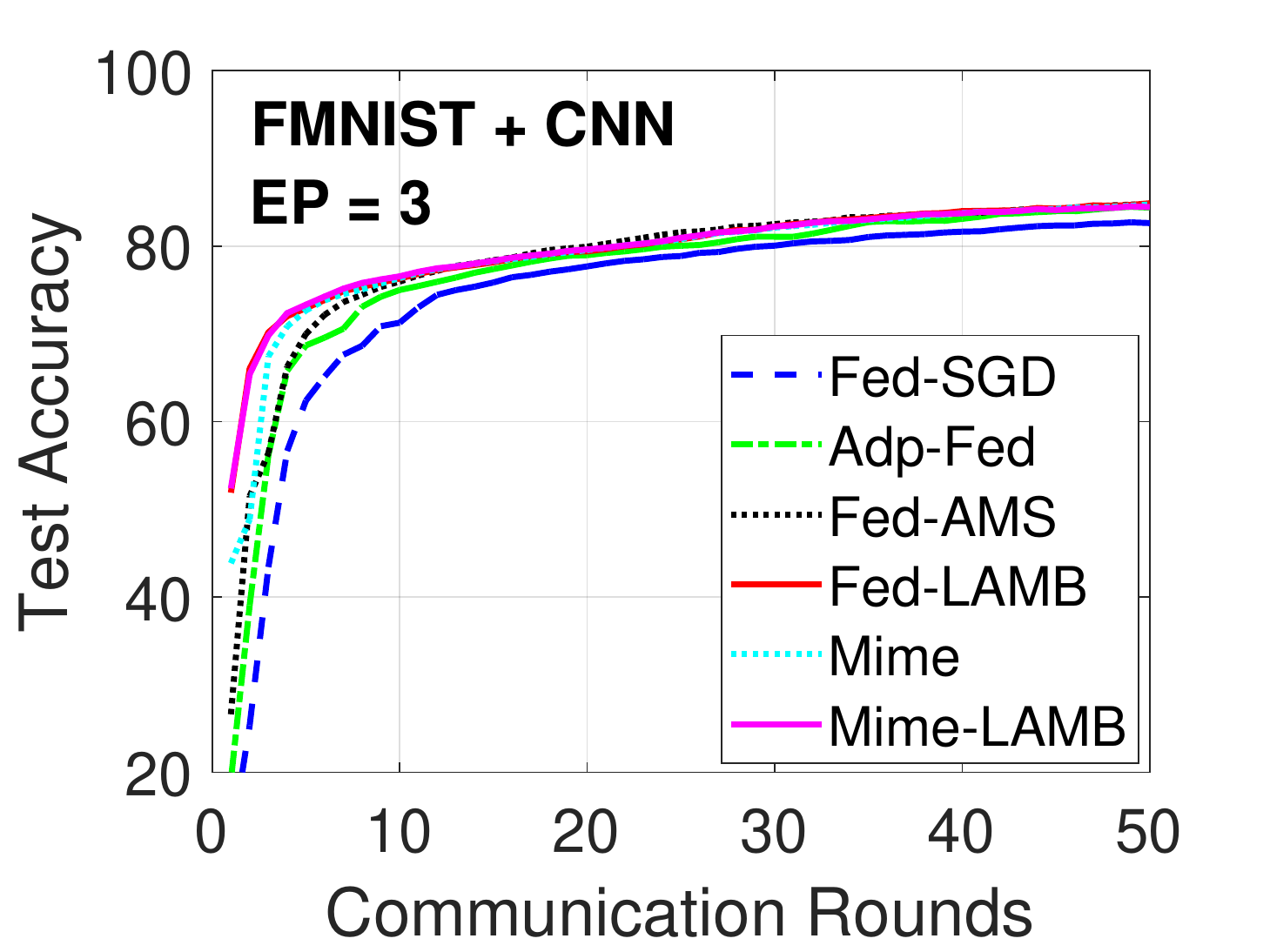}
        }
    \end{center}

\vspace{-0.2in}

	\caption{\textbf{IID data setting}. Test accuracy against the number of communication rounds.
	}
	\label{fig:iid}
\end{figure}

In Figure~\ref{fig:iid}, we report the test accuracy of MLP trained on MNIST, as well as CNN trained on MNIST and FMNIST, where the data are IID allocated among the clients. We test 1 local epoch and 3 local epochs (more local iterations). In all the figures, we observe a clear advantage of Fed-LAMB over the competing methods in terms of the convergence speed. In particular, we can see that Fed-LAMB is able to achieve the same accuracy with fewest number of communication rounds, thus improving the communication efficiency. For instance, this can be observed as follows: on MNIST + CNN (1 local epoch), Fed-AMS requires 20 rounds to achieves 90\% accuracy, while Fed-LAMB only takes 5 rounds. This implies a 75\% reduction in the communication cost. Moreover, on MNIST, Fed-LAMB also leads to improved generalisation performance, i.e., test accuracy. We can draw same conclusions with 3 local epochs. Also, similar comparison holds for Mime-LAMB vs. Mime. In general, the Mime-LAMB variant matches the performance of Fed-LAMB closely.

\subsection{Comparison under non-IID settings}

\begin{figure}[h]
    \begin{center}
        \mbox{
        \hspace{-0.15in}\includegraphics[width=2.25in]{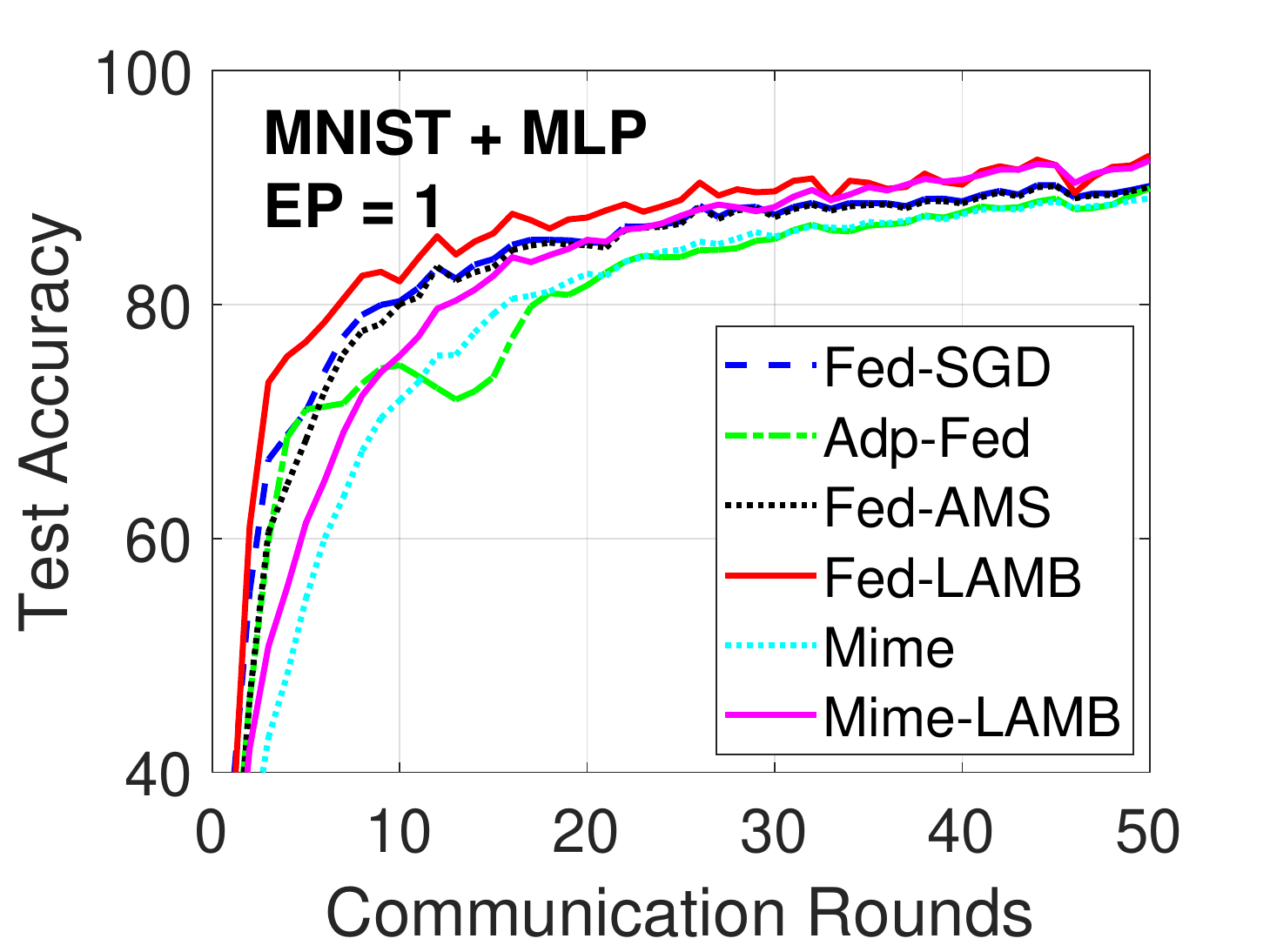}\hspace{-0.1in}
        \includegraphics[width=2.25in]{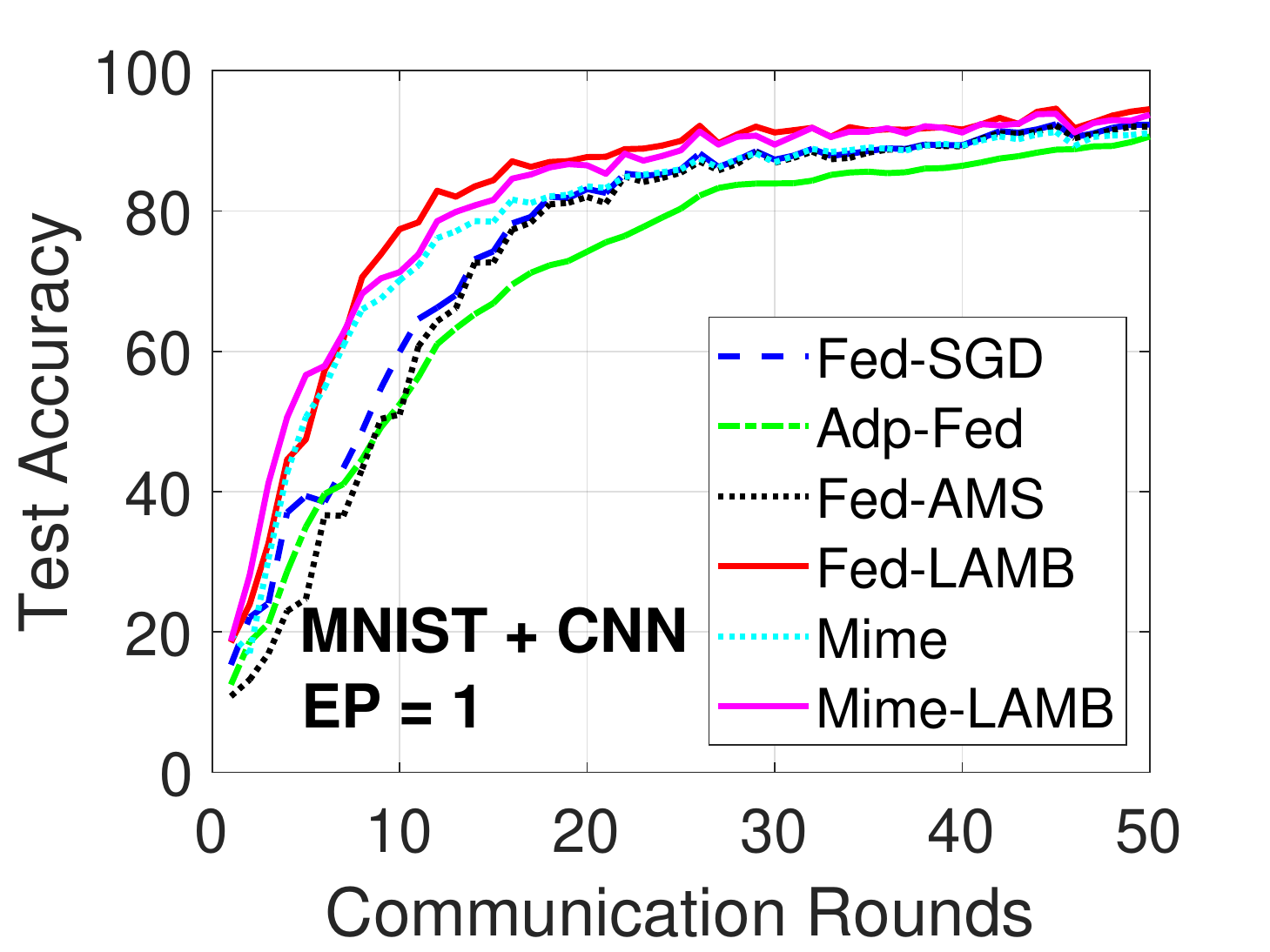}
        \hspace{-0.1in}
        \includegraphics[width=2.25in]{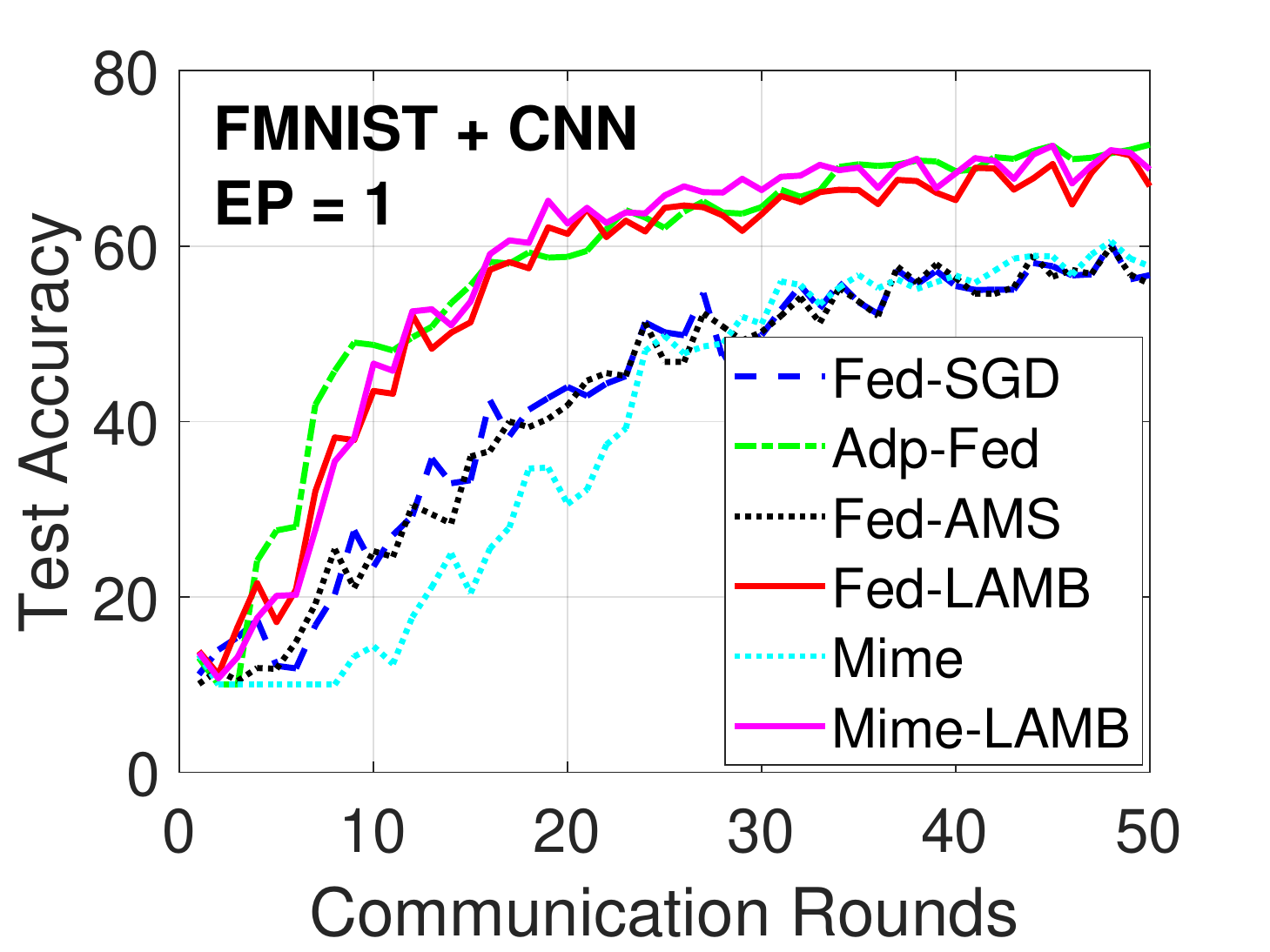}
        }
        \mbox{
        \hspace{-0.15in}\includegraphics[width=2.25in]{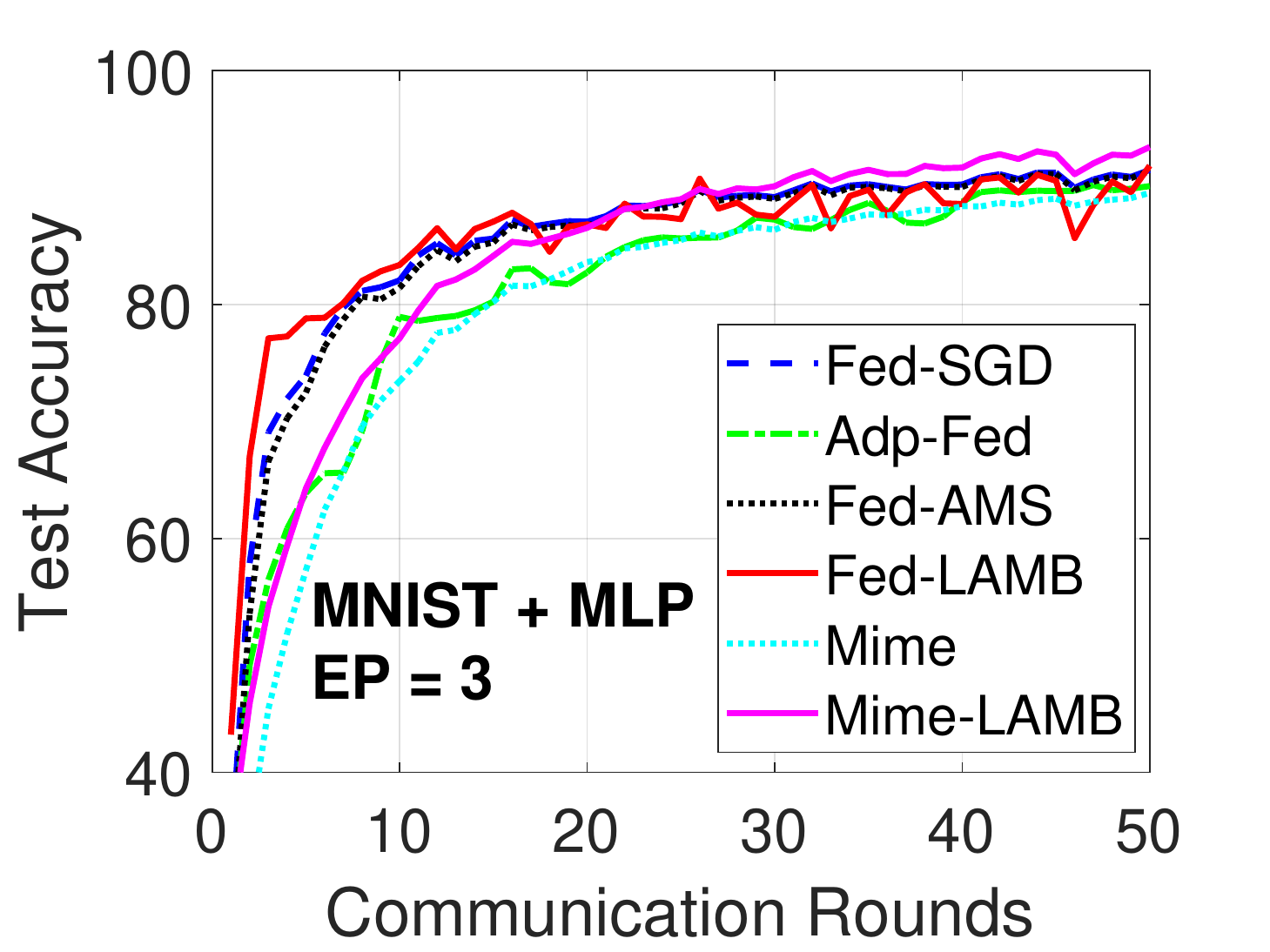}\hspace{-0.1in}
        \includegraphics[width=2.25in]{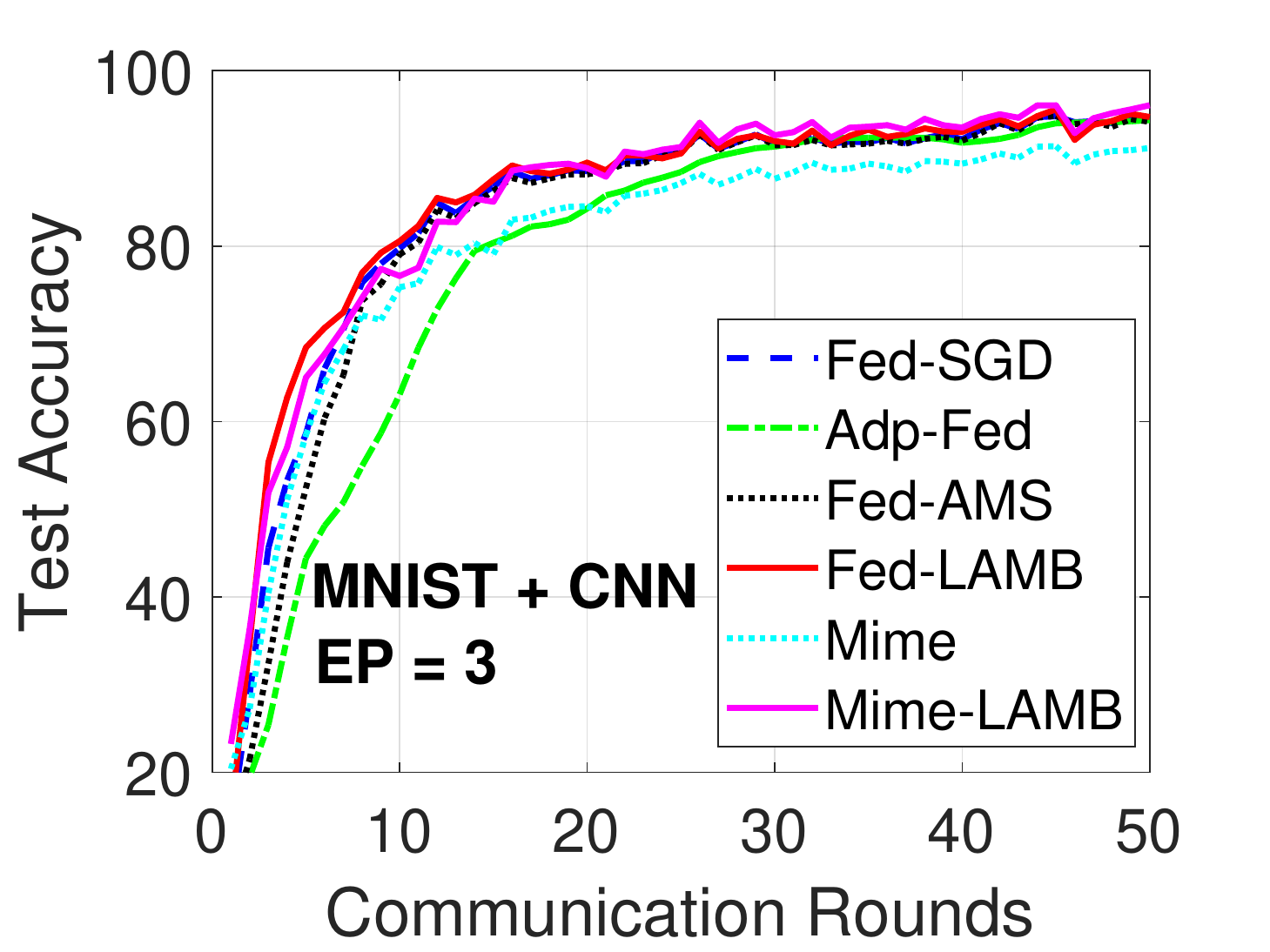}
        \hspace{-0.1in}
        \includegraphics[width=2.25in]{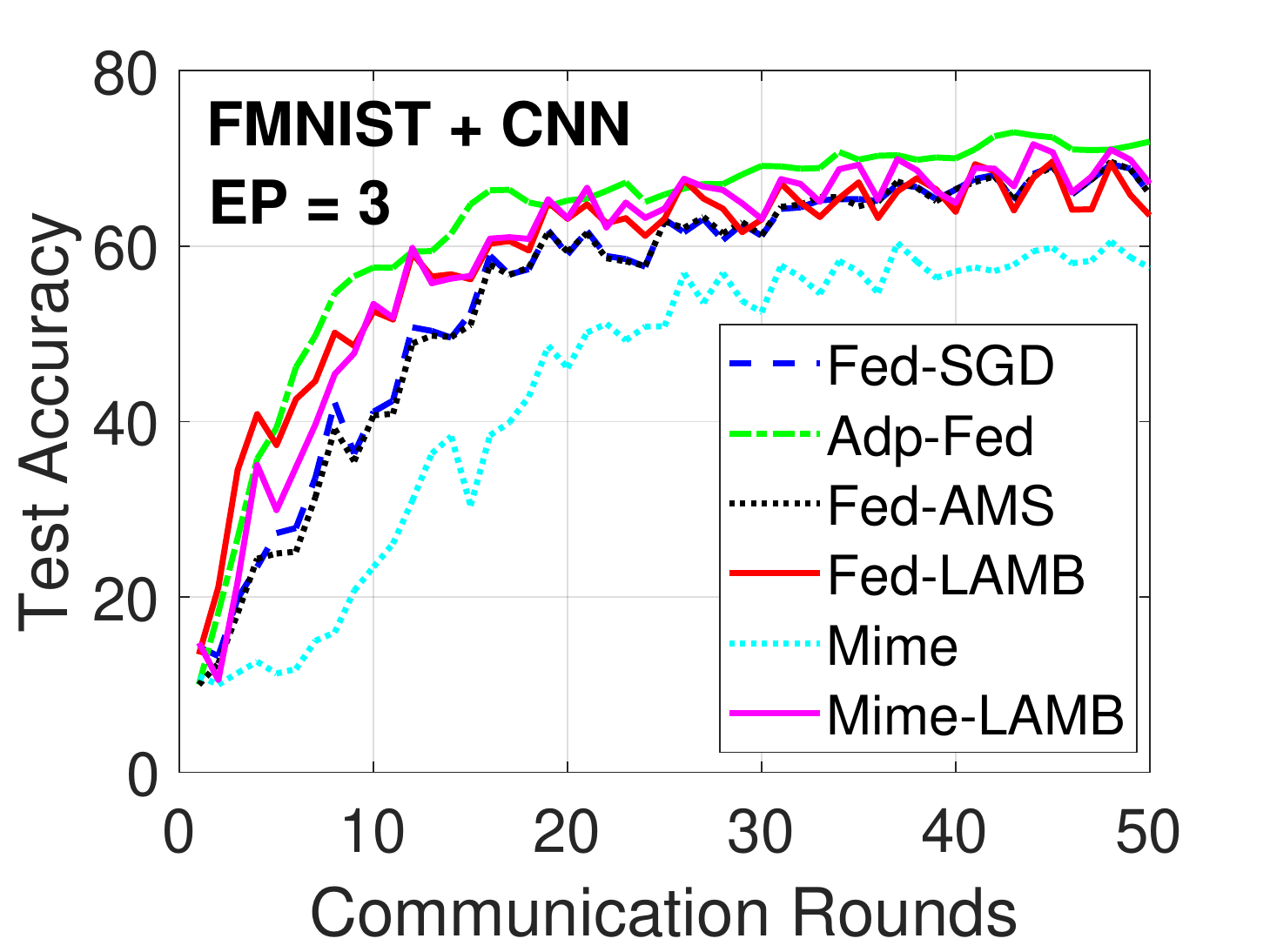}
        }
    \end{center}

\vspace{-0.2in}

	\caption{\textbf{non-IID data setting.} Test accuracy against the number of communication rounds.}
	\label{fig:noniid}	\vspace{0.2in}
\end{figure}

In Figure~\ref{fig:noniid}, we provide the results on  MNIST and FMNIST with non-IID local data distribution. In particular, in each round of federated training, every local device only receives samples from one or two class (out of ten). We see that for experiments with 1 local epoch, in all cases our proposed Fed-LAMB outperforms all the baseline methods. Similar to the IID data setting, Fed-LAMB provides faster convergence speed and achieves higher test accuracy than Fed-SGD and Fed-AMS. The advantage is especially significant for the CNN model, e.g., it improves the accuracy of Fed-SGD and Fed-AMS by more than 10\% on FMNIST at the 50-th round. The other baseline method, Adp-Fed, performs as good as our Fed-LAMB on FMNIST, but worse than other methods on MNIST. Mime-LAMB also considerably improves Mime on all the tasks, see Figure~\ref{fig:noniid}. In general, the two variants Fed-LAMB and Mime-LAMB perform similarly.

\vspace{0.1in}

The relative comparison is basically the same when the local models are trained for 3 epochs before model aggregation. The proposed methods converge faster than the underlying baselines. Adp-Fed performs well on FMNIST but worse on MNIST. Yet, we notice that the advantage of Fed-LAMB becomes less significant than what we observed in Figure~\ref{fig:iid} with IID data. One plausible reason is that when the local data is highly non-IID, the local learning objective functions become very different. Intuitively, with more local steps, learning the local models fast might not always do good to the global model, as local models target at different loss functions. This is consistent with extensive empirical observations in literature that in FL, allowing too many local training steps usually pushes local models too far which may hinder the convergence of the global model.

\newpage

\begin{figure}[h]

  \begin{center}
        \mbox{
        \includegraphics[width=2.5in]{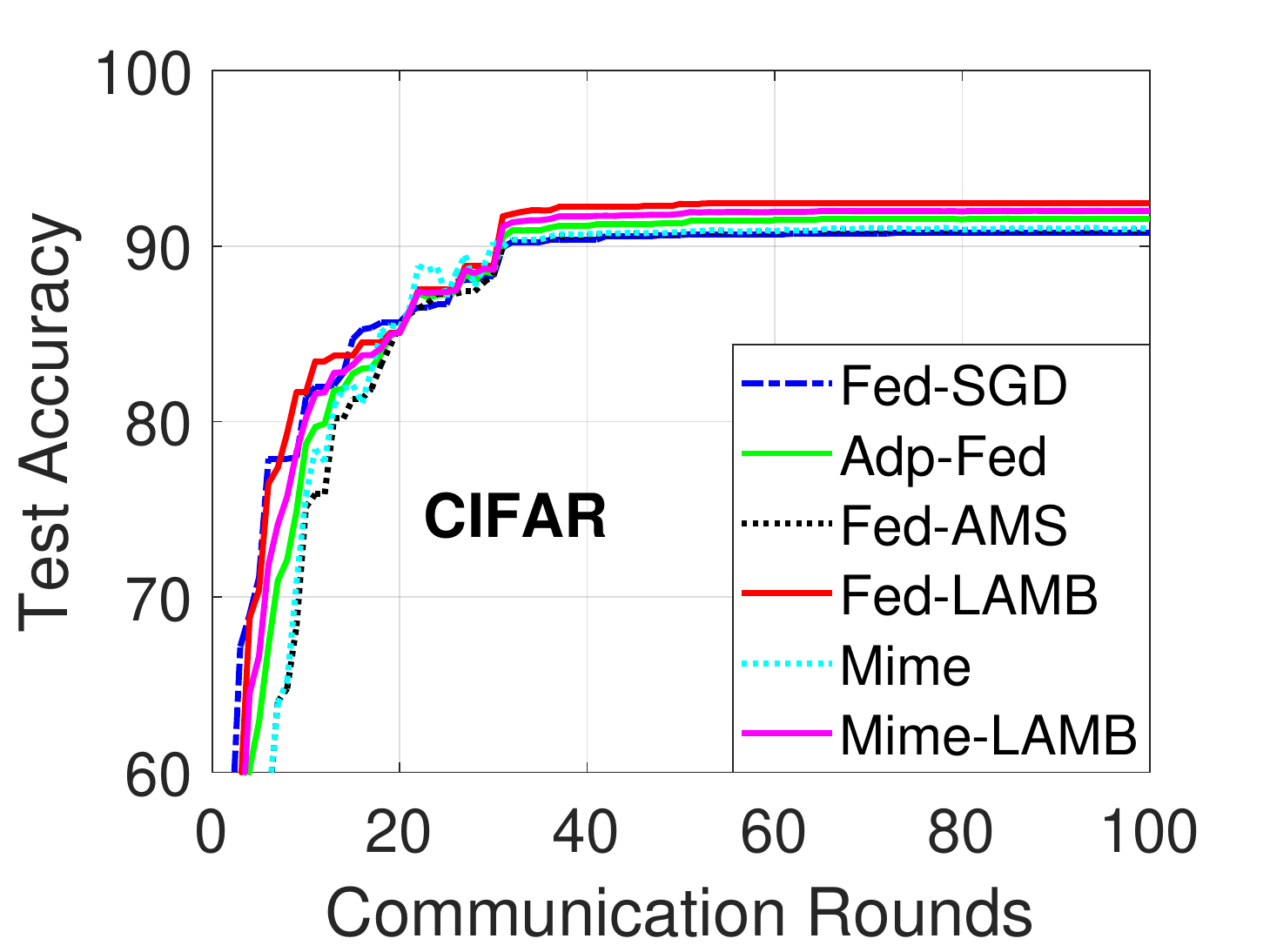}
        \includegraphics[width=2.5in]{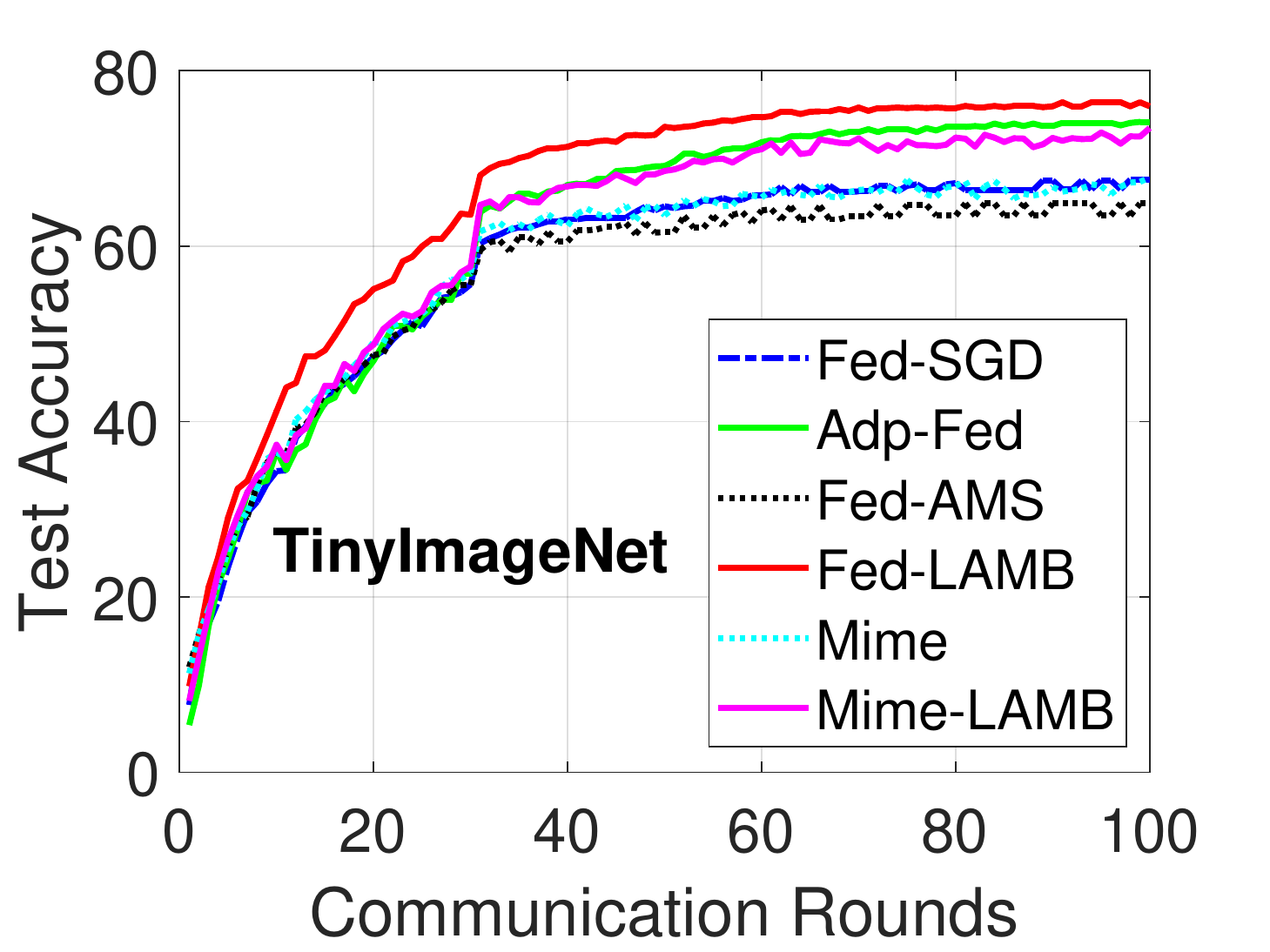}
        }
    \end{center}

\vspace{-0.2in}

	\caption{\textbf{non-IID data.} Test accuracy of CIFAR-10 and TinyImagenet on ResNet-18.
	}
	\label{fig:noniidresnet18}
\end{figure}

In Figure~\ref{fig:noniidresnet18}, we present the results on CIFAR-10 and TinyImageNet datasets trained by ResNet-18. When training these two models, we decrease the learning rate to $1/10$ at the 30-th and 70-th communication round. From Figure~\ref{fig:noniidresnet18}, we can draw similar conclusion as before: the proposed Fed-LAMB is the best method in terms of both convergence speed and generalization accuracy. In particular, on TinyImageNet, we see that Fed-LAMB has a significant advantage over all three baselines. Although Adp-Fed performs better than Fed-SGD and Fed-AMS, it is considerably worse than Fed-LAMB. We report the test accuracy at the end of training in Table~\ref{tab:acc}. Fed-LAMB achieves the highest accuracy on both datasets. Mime-LAMB also substantially improves Mime.

\begin{table}[h]
\centering
\resizebox{1\columnwidth}{!}{%
\begin{tabular}{c|cccccc}
\toprule[1pt]
 & Fed-SGD    & Adp-Fed    & Fed-AMS    & \textbf{Fed-LAMB }   & Mime & \textbf{Mime-LAMB}        \\ \hline
CIFAR-10 & 90.75 $\pm$ 0.48  &91.57 $\pm$ 0.38  & 90.93 $\pm$ 0.22 &  \textbf{92.44 $\pm$ 0.53} & 90.94 $\pm$ 0.13 & \textbf{92.00 $\pm$ 0.21}  \\
TinyImageNet & 67.58 $\pm$ 0.21  &  74.17 $\pm$ 0.43   & 64.86 $\pm$ 0.83& \textbf{76.00 $\pm$ 0.26} & 67.82 $\pm$ 0.24 & \textbf{73.46 $\pm$ 0.25} \\
\toprule[1pt]
\end{tabular}
}
\caption{Test accuracy with ResNet-18 network after 100 communication rounds.}
\label{tab:acc}
\end{table}

\begin{figure}[h]
  \begin{center}
  \mbox{
        \includegraphics[width=2.5in]{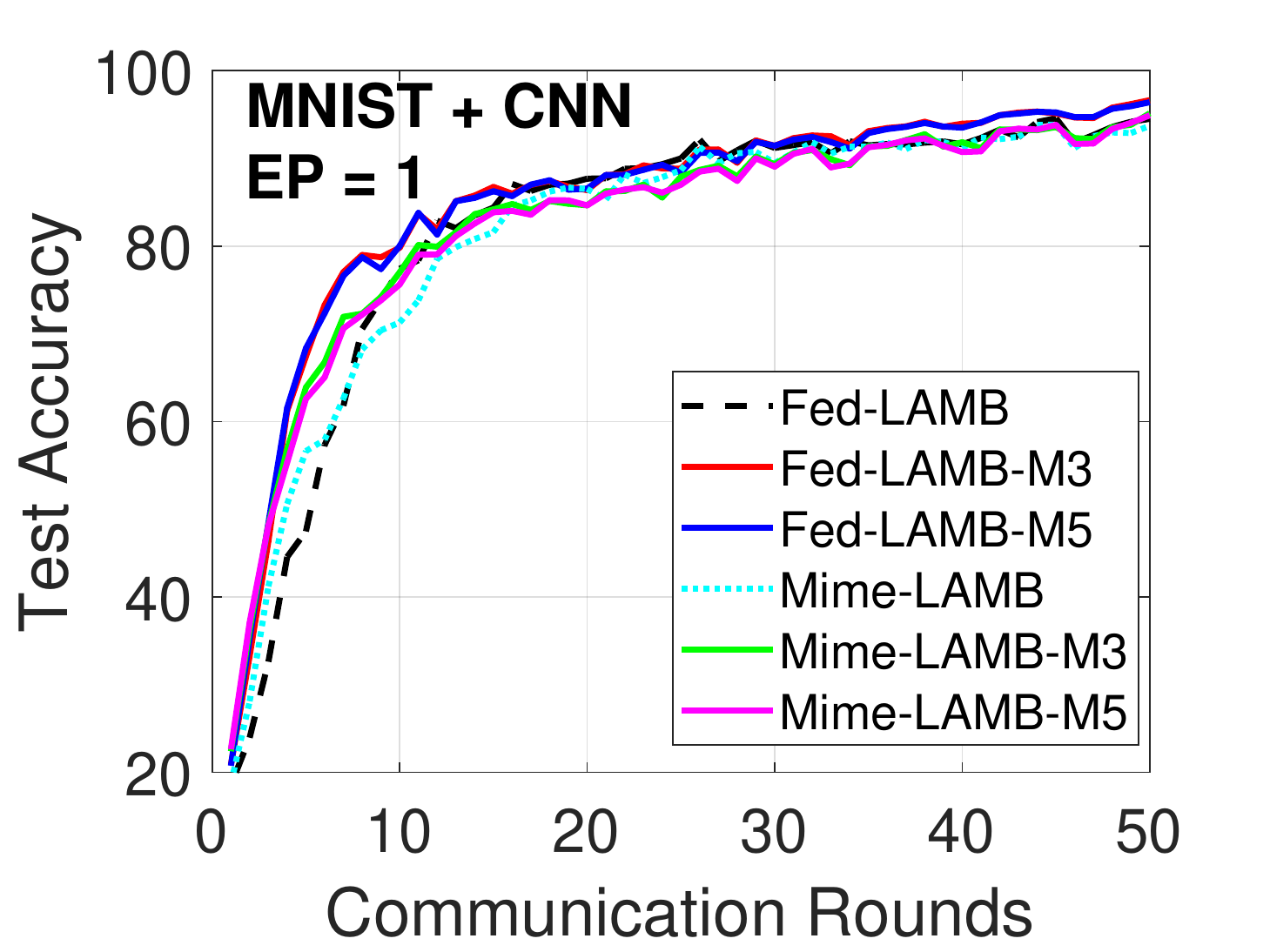}
        \includegraphics[width=2.5in]{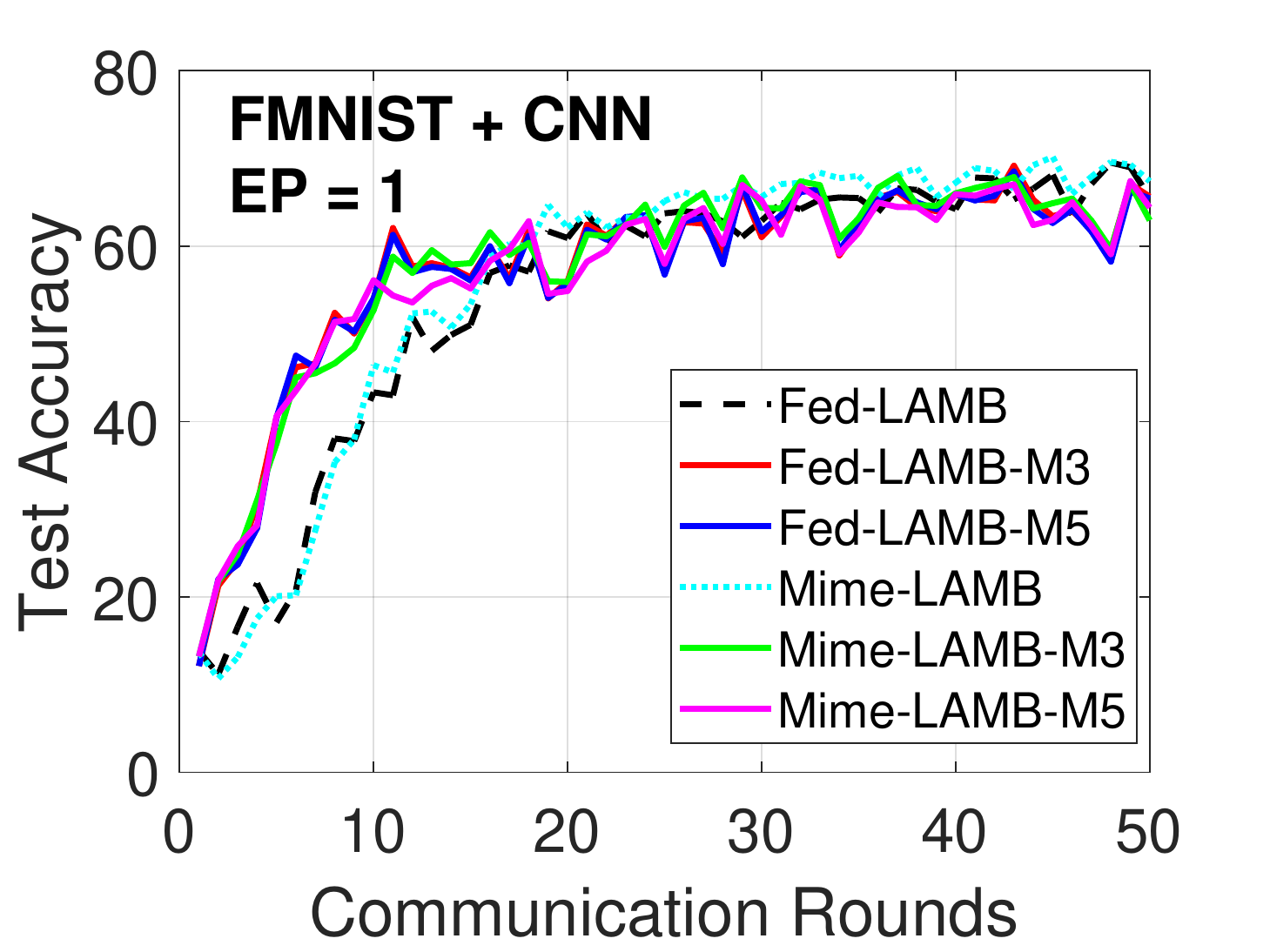}
        }
    \end{center}

\vspace{-0.2in}

	\caption{\textbf{non-IID data.} Fed-LAMB and Mime-Fed with lazy synchronization of $\hat v$.
	}
	\label{fig:lazy}

\end{figure}

In addition, in Figure~\ref{fig:lazy} we present the result of our methods with lazy synchronization of $\hat v$, where the server updates and broadcasts $\hat v$ every $Z=3,5$ rounds, instead of very single round. We see that practically the performance is similar to the standard algorithm with every-round global $\hat v$ updates (sometimes even slightly better).

\newpage

\noindent\textbf{More Workers:} In Figure~\ref{fig:client200}, we provide additional figures with larger number of workers $n=200$, on MNIST and FMNIST with non-IID data. The conclusions stay the same: we see that the proposed Fed-LAMB and Mime-LAMB perform much better than the baseline algorithms, with faster convergence and better accuracy at the end of 100 FL training rounds.

\begin{figure}[t]
  \begin{center}
  \mbox{
    \includegraphics[width=2.5in]{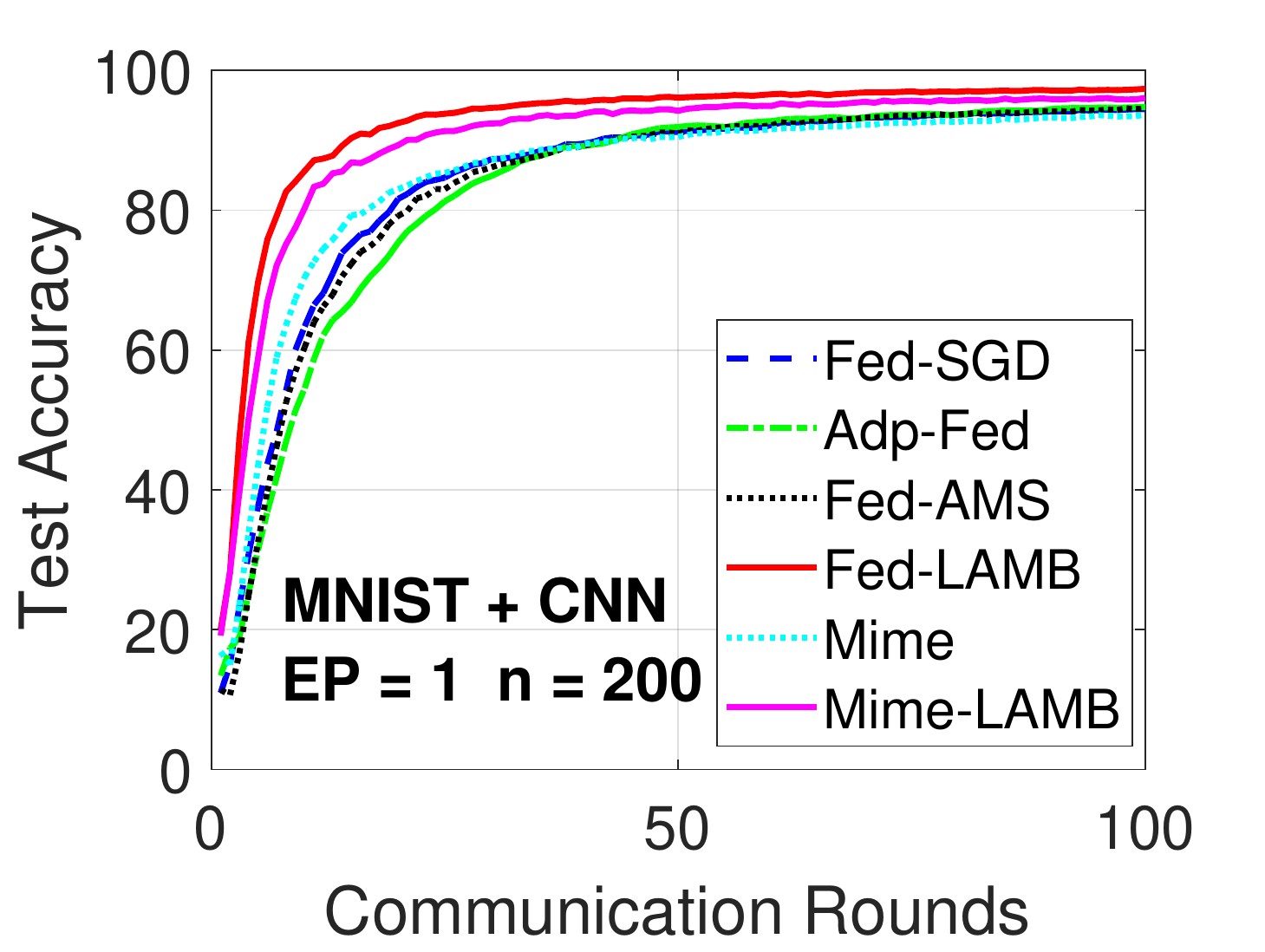}
    \includegraphics[width=2.5in]{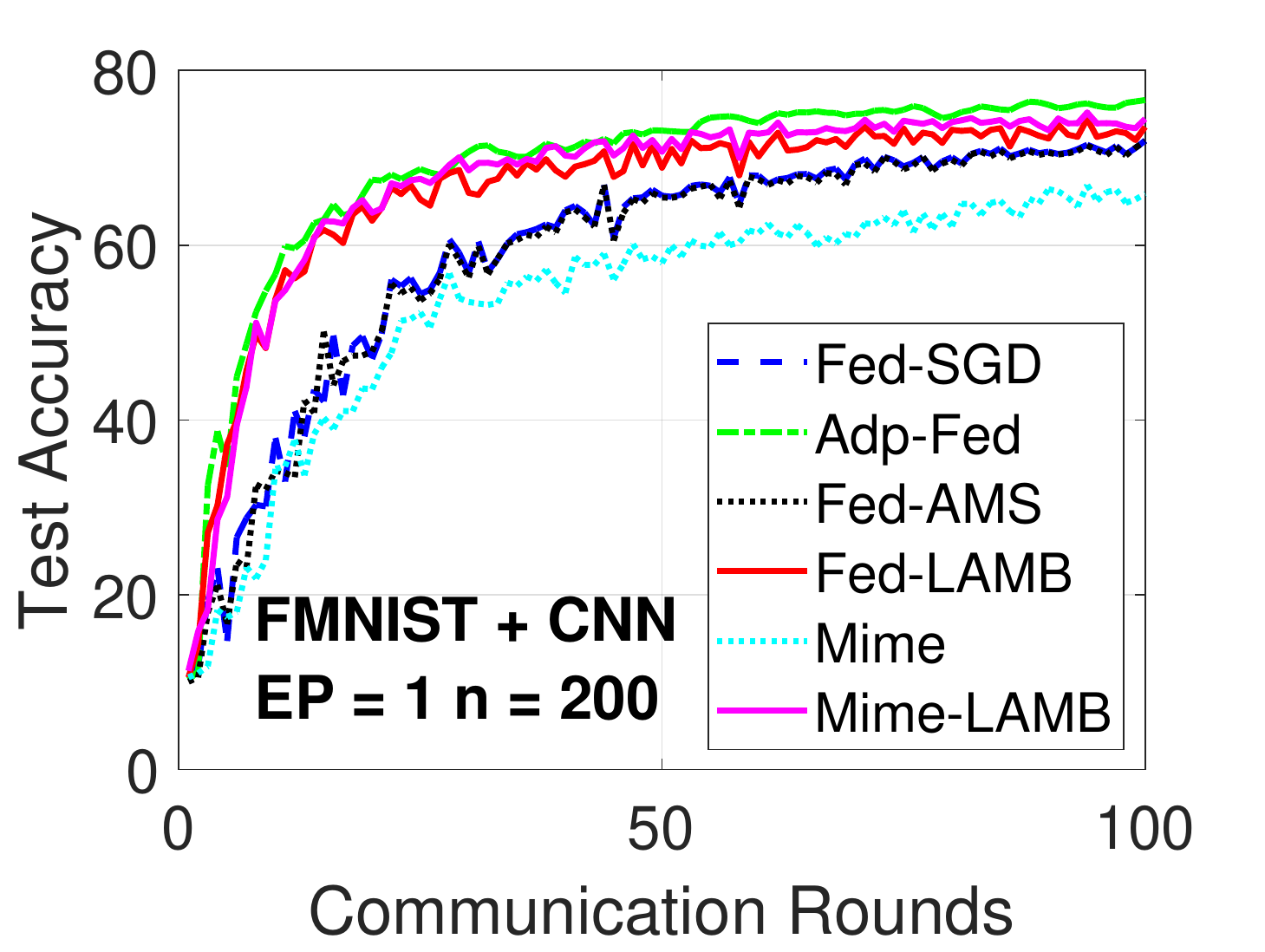}
    }
  \caption{Test accuracy with $n=200$ workers, full participation, local batch size 64. Data are non-IID distributed among clients.}
  \label{fig:client200}
  \end{center}
\end{figure}

\subsection{Summary of empirical findings}

We provide a brief summary of this section. On all datasets, the primary comparison of most importance appears evident:
\begin{align*}
    \textbf{Fed-LAMB$\approx$ Mime-LAMB$>$Fed-AMS$\approx$Mime.}
\end{align*}
The proposed scheme (with two variants) exhibits faster convergence and better generalisation accuracy than recently proposed adaptive FL algorithms. Our results suggest that, similar as in the single-machine training, layer-wise acceleration can also be effective and beneficial in federated learning. Moreover, in practice we may adopt the lazy aggregation trick to further reduce the additional communication required for Fed-LAMB.

\section{Conclusion}\label{sec:conclusion}

We study a doubly adaptive method in the particular framework of federated learning (FL). Built upon the acceleration effect of layer-wise learning rate scheduling and of state-of-the-art adaptive gradient methods, we derive a locally layer-wise FL framework that performs local updates using adaptive AMSGrad on each worker and periodically averages local models stored on each device.
The core of our Fed-LAMB scheme, is to speedup up local training by adopting layer-wise adaptive
learning rates. To out knowledge, this is the first FL algorithm in literature that possess both the  \emph{dimension-wise} adaptivity (by AMSGrad) and \emph{layer-wise} adaptivity (by layer-wise learning rate). We provide the convergence analysis of Fed-LAMB that matches many existing methods, with a linear speedup against the number of clients.
Extensive experiments on various datasets and models, under both IID and non-IID data settings, validate that both Fed-LAMB and Mime-LAMB are able to provide faster convergence  which in turn could lead to reduced communication cost.
In many cases, our framework also improves the overall performance~of~federated~learning~over~prior~methods.

\bibliographystyle{plainnat}
\bibliography{ref}

\begin{thebibliography}{41}
\providecommand{\natexlab}[1]{#1}
\providecommand{\url}[1]{\texttt{#1}}
\expandafter\ifx\csname urlstyle\endcsname\relax
  \providecommand{\doi}[1]{doi: #1}\else
  \providecommand{\doi}{doi: \begingroup \urlstyle{rm}\Url}\fi

\bibitem[Bonawitz et~al.(2019)Bonawitz, Eichner, Grieskamp, Huba, Ingerman,
  Ivanov, Kiddon, Kone{\v{c}}n{\'y}, Mazzocchi, McMahan, Overveldt, Petrou,
  Ramage, and Roselander]{bonawitz2019towards}
Kallista~A. Bonawitz, Hubert Eichner, Wolfgang Grieskamp, Dzmitry Huba, Alex
  Ingerman, Vladimir Ivanov, Chlo{\'{e}} Kiddon, Jakub Kone{\v{c}}n{\'y},
  Stefano Mazzocchi, Brendan McMahan, Timon~Van Overveldt, David Petrou, Daniel
  Ramage, and Jason Roselander.
\newblock Towards federated learning at scale: System design.
\newblock In \emph{Proceedings of Machine Learning and Systems (MLSys)},
  Stanford, CA, 2019.

\bibitem[Chen et~al.(2019)Chen, Liu, Sun, and Hong]{Proc:Chen_ICLR19}
Xiangyi Chen, Sijia Liu, Ruoyu Sun, and Mingyi Hong.
\newblock On the convergence of {A} class of adam-type algorithms for
  non-convex optimization.
\newblock In \emph{Proceedings of the 7th International Conference on Learning
  Representations (ICLR)}, New Orleans, LA, 2019.

\bibitem[Chen et~al.(2020)Chen, Li, and Li]{chen2020toward}
Xiangyi Chen, Xiaoyun Li, and Ping Li.
\newblock Toward communication efficient adaptive gradient method.
\newblock In \emph{Proceedings of the {ACM-IMS} Foundations of Data Science
  Conference (FODS)}, pages 119--128, Virtual Event, USA, 2020.

\bibitem[Deng et~al.(2009)Deng, Dong, Socher, Li, Li, and
  Fei{-}Fei]{deng2009imagenet}
Jia Deng, Wei Dong, Richard Socher, Li{-}Jia Li, Kai Li, and Li~Fei{-}Fei.
\newblock {ImageNet}: {A} large-scale hierarchical image database.
\newblock In \emph{Proceedings of the 2009 {IEEE} Computer Society Conference
  on Computer Vision and Pattern Recognition (CVPR)}, pages 248--255, Miami,
  FL, 2009.

\bibitem[Devlin et~al.(2019)Devlin, Chang, Lee, and Toutanova]{bert19}
Jacob Devlin, Ming{-}Wei Chang, Kenton Lee, and Kristina Toutanova.
\newblock {BERT:} pre-training of deep bidirectional transformers for language
  understanding.
\newblock In \emph{Proceedings of the 2019 Conference of the North American
  Chapter of the Association for Computational Linguistics: Human Language
  Technologies ({NAACL-HLT})}, pages 4171--4186, Minneapolis, MN, 2019.

\bibitem[Dozat(2016)]{dozat2016incorporating}
Timothy Dozat.
\newblock Incorporating nesterov momentum into {Adam}.
\newblock In \emph{Proceedings of the 4th International Conference on Learning
  Representations (ICLR Workshop)}, San Juan, Puerto Rico, 2016.

\bibitem[Duchi et~al.(2011)Duchi, Hazan, and Singer]{DHS11}
John~C. Duchi, Elad Hazan, and Yoram Singer.
\newblock Adaptive subgradient methods for online learning and stochastic
  optimization.
\newblock \emph{J. Mach. Learn. Res.}, 12:\penalty0 2121--2159, 2011.

\bibitem[He et~al.(2016)He, Zhang, Ren, and Sun]{Proc:He-resnet16}
Kaiming He, Xiangyu Zhang, Shaoqing Ren, and Jian Sun.
\newblock Deep residual learning for image recognition.
\newblock In \emph{Proceedings of the 2016 {IEEE} Conference on Computer Vision
  and Pattern Recognition (CVPR)}, pages 770--778, Las Vegas, NV, 2016.

\bibitem[Karimireddy et~al.(2019)Karimireddy, Kale, Mohri, Reddi, Stich, and
  Suresh]{karimireddy2019scaffold}
Sai~Praneeth Karimireddy, Satyen Kale, Mehryar Mohri, Sashank~J Reddi,
  Sebastian~U Stich, and Ananda~Theertha Suresh.
\newblock Scaffold: Stochastic controlled averaging for on-device federated
  learning.
\newblock \emph{arXiv preprint arXiv:1910.06378}, 2019.

\bibitem[Karimireddy et~al.(2020)Karimireddy, Jaggi, Kale, Mohri, Reddi, Stich,
  and Suresh]{karimireddy2020mime}
Sai~Praneeth Karimireddy, Martin Jaggi, Satyen Kale, Mehryar Mohri, Sashank~J
  Reddi, Sebastian~U Stich, and Ananda~Theertha Suresh.
\newblock Mime: Mimicking centralized stochastic algorithms in federated
  learning.
\newblock \emph{arXiv preprint arXiv:2008.03606}, 2020.

\bibitem[Khaled et~al.(2020)Khaled, Mishchenko, and
  Richt{\'{a}}rik]{Proc:Khaled_AISTATS20}
Ahmed Khaled, Konstantin Mishchenko, and Peter Richt{\'{a}}rik.
\newblock Tighter theory for local {SGD} on identical and heterogeneous data.
\newblock In \emph{Proceedings of the 23rd International Conference on
  Artificial Intelligence and Statistics (AISTATS)}, pages 4519--4529, Online
  [Palermo, Sicily, Italy], 2020.

\bibitem[Kingma and Ba(2015)]{KB15}
Diederik~P. Kingma and Jimmy Ba.
\newblock Adam: {A} method for stochastic optimization.
\newblock In \emph{Proceedings of the 3rd International Conference on Learning
  Representations (ICLR)}, San Diego, CA, 2015.

\bibitem[Kone{\v{c}}n{\`y} et~al.(2016)Kone{\v{c}}n{\`y}, McMahan, Yu,
  Richt{\'a}rik, Suresh, and Bacon]{konevcny2016federated}
Jakub Kone{\v{c}}n{\`y}, H~Brendan McMahan, Felix~X Yu, Peter Richt{\'a}rik,
  Ananda~Theertha Suresh, and Dave Bacon.
\newblock Federated learning: Strategies for improving communication
  efficiency.
\newblock \emph{arXiv preprint arXiv:1610.05492}, 2016.

\bibitem[Krizhevsky(2009)]{krizhevsky2009learning}
Alex Krizhevsky.
\newblock Learning multiple layers of features from tiny images.
\newblock \emph{Master's thesis, Department of Computer Science, University of
  Toronto}, 2009.

\bibitem[LeCun(1998)]{lecun1998mnist}
Yann LeCun.
\newblock The mnist database of handwritten digits.
\newblock \emph{http://yann. lecun. com/exdb/mnist/}, 1998.

\bibitem[Leroy et~al.(2019)Leroy, Coucke, Lavril, Gisselbrecht, and
  Dureau]{Proc:Leroy_ICASSP19}
David Leroy, Alice Coucke, Thibaut Lavril, Thibault Gisselbrecht, and Joseph
  Dureau.
\newblock Federated learning for keyword spotting.
\newblock In \emph{Proceedings of the {IEEE} International Conference on
  Acoustics, Speech and Signal Processing (ICASSP)}, pages 6341--6345,
  Brighton, UK, 2019.

\bibitem[Li et~al.(2020{\natexlab{a}})Li, Sahu, Talwalkar, and
  Smith]{li2019federated}
Tian Li, Anit~Kumar Sahu, Ameet Talwalkar, and Virginia Smith.
\newblock Federated learning: Challenges, methods, and future directions.
\newblock \emph{{IEEE} Signal Process. Mag.}, 37\penalty0 (3):\penalty0 50--60,
  2020{\natexlab{a}}.

\bibitem[Li et~al.(2020{\natexlab{b}})Li, Huang, Yang, Wang, and
  Zhang]{Proc:Li_ICLR20}
Xiang Li, Kaixuan Huang, Wenhao Yang, Shusen Wang, and Zhihua Zhang.
\newblock On the convergence of fedavg on non-iid data.
\newblock In \emph{Proceedings of the 8th International Conference on Learning
  Representations (ICLR)}, Addis Ababa, Ethiopia, 2020{\natexlab{b}}.

\bibitem[Li et~al.(2022)Li, Karimi, and Li]{Proc:Li_ICLR22}
Xiaoyun Li, Belhal Karimi, and Ping Li.
\newblock On distributed adaptive optimization with gradient compression.
\newblock In \emph{Proceedings of the 10th International Conference on Learning
  Representations (ICLR)}, Virtual Event, 2022.

\bibitem[Liang et~al.(2019)Liang, Shen, Liu, Pan, Chen, and
  Cheng]{liang2019variance}
Xianfeng Liang, Shuheng Shen, Jingchang Liu, Zhen Pan, Enhong Chen, and Yifei
  Cheng.
\newblock Variance reduced local sgd with lower communication complexity.
\newblock \emph{arXiv preprint arXiv:1912.12844}, 2019.

\bibitem[McMahan and Streeter(2010)]{mcmahan2010adaptive}
Brendan McMahan and Matthew~J. Streeter.
\newblock Adaptive bound optimization for online convex optimization.
\newblock In \emph{Proceedings of the 23rd Conference on Learning Theory
  (COLT)}, pages 244--256, Haifa, Israel, 2010.

\bibitem[McMahan et~al.(2017)McMahan, Moore, Ramage, Hampson, and
  y~Arcas]{mcmahan2017communication}
Brendan McMahan, Eider Moore, Daniel Ramage, Seth Hampson, and
  Blaise~Ag{\"{u}}era y~Arcas.
\newblock Communication-efficient learning of deep networks from decentralized
  data.
\newblock In \emph{Proceedings of the 20th International Conference on
  Artificial Intelligence and Statistics (AISTATS)}, pages 1273--1282, Fort
  Lauderdale, FL, 2017.

\bibitem[Nesterov(2004)]{N04}
Yurii Nesterov.
\newblock Introductory lectures on convex optimization: A basic course.
\newblock \emph{Springer}, 2004.

\bibitem[Niknam et~al.(2020)Niknam, Dhillon, and Reed]{Article:NiknamDR20}
Solmaz Niknam, Harpreet~S. Dhillon, and Jeffrey~H. Reed.
\newblock Federated learning for wireless communications: Motivation,
  opportunities, and challenges.
\newblock \emph{{IEEE} Commun. Mag.}, 58\penalty0 (6):\penalty0 46--51, 2020.

\bibitem[Polyak(1964)]{P64}
B.~T. Polyak.
\newblock Some methods of speeding up the convergence of iteration methods.
\newblock \emph{Mathematics and Mathematical Physics}, 1964.

\bibitem[Reddi et~al.(2018)Reddi, Kale, and Kumar]{reddi2019convergence}
Sashank~J. Reddi, Satyen Kale, and Sanjiv Kumar.
\newblock On the convergence of adam and beyond.
\newblock In \emph{Proceedings of the 6th International Conference on Learning
  Representations (ICLR)}, Vancouver, Canada, 2018.

\bibitem[Reddi et~al.(2021)Reddi, Charles, Zaheer, Garrett, Rush,
  Kone{\v{c}}n{\'y}, Kumar, and McMahan]{reddi2020adaptive}
Sashank~J. Reddi, Zachary Charles, Manzil Zaheer, Zachary Garrett, Keith Rush,
  Jakub Kone{\v{c}}n{\'y}, Sanjiv Kumar, and Hugh~Brendan McMahan.
\newblock Adaptive federated optimization.
\newblock In \emph{Proceedings of the 9th International Conference on Learning
  Representations (ICLR)}, Virtual Event, Austria, 2021.

\bibitem[Sahu et~al.(2018)Sahu, Li, Sanjabi, Zaheer, Talwalkar, and
  Smith]{Article:Sahu_arxiv18}
Anit~Kumar Sahu, Tian Li, Maziar Sanjabi, Manzil Zaheer, Ameet Talwalkar, and
  Virginia Smith.
\newblock On the convergence of federated optimization in heterogeneous
  networks.
\newblock \emph{arXiv preprint arXiv:1812.06127}, 2018.

\bibitem[Tieleman and Hinton(2012)]{TH12}
T.~Tieleman and G.~Hinton.
\newblock Rmsprop: Divide the gradient by a running average of its recent
  magnitude.
\newblock \emph{COURSERA: Neural Networks for Machine Learning}, 2012.

\bibitem[Wang et~al.(2020)Wang, Tantia, Ballas, and
  Rabbat]{Proc:WangTBR_ICLR20}
Jianyu Wang, Vinayak Tantia, Nicolas Ballas, and Michael~G. Rabbat.
\newblock Slowmo: Improving communication-efficient distributed {SGD} with slow
  momentum.
\newblock In \emph{Proceedings of the 8th International Conference on Learning
  Representations (ICLR)}, Addis Ababa, Ethiopia, 2020.

\bibitem[Woodworth et~al.(2020)Woodworth, Patel, Stich, Dai, Bullins, McMahan,
  Shamir, and Srebro]{Proc:Woodworth_ICML20}
Blake~E. Woodworth, Kumar~Kshitij Patel, Sebastian~U. Stich, Zhen Dai, Brian
  Bullins, H.~Brendan McMahan, Ohad Shamir, and Nathan Srebro.
\newblock Is local {SGD} better than minibatch sgd?
\newblock In \emph{Proceedings of the 37th International Conference on Machine
  Learning (ICML)}, pages 10334--10343, Virtual Event, 2020.

\bibitem[Xiao et~al.(2017)Xiao, Rasul, and Vollgraf]{xiao2017fashion}
Han Xiao, Kashif Rasul, and Roland Vollgraf.
\newblock {Fashion-MNIST}: a novel image dataset for benchmarking machine
  learning algorithms.
\newblock \emph{arXiv preprint arXiv:1708.07747}, 2017.

\bibitem[Xu et~al.(2021)Xu, Glicksberg, Su, Walker, Bian, and
  Wang]{Proc:XuGSWBW21}
Jie Xu, Benjamin~S. Glicksberg, Chang Su, Peter~B. Walker, Jiang Bian, and Fei
  Wang.
\newblock Federated learning for healthcare informatics.
\newblock \emph{J. Heal. Informatics Res.}, 5\penalty0 (1):\penalty0 1--19,
  2021.

\bibitem[Yang et~al.(2019)Yang, Liu, Chen, and Tong]{Article:YangLCT19}
Qiang Yang, Yang Liu, Tianjian Chen, and Yongxin Tong.
\newblock Federated machine learning: Concept and applications.
\newblock \emph{{ACM} Trans. Intell. Syst. Technol.}, 10\penalty0 (2):\penalty0
  12:1--12:19, 2019.

\bibitem[You et~al.(2018)You, Zhang, Hsieh, Demmel, and Keutzer]{Proc:LARS18}
Yang You, Zhao Zhang, Cho{-}Jui Hsieh, James Demmel, and Kurt Keutzer.
\newblock Imagenet training in minutes.
\newblock In \emph{Proceedings of the 47th International Conference on Parallel
  Processing (ICPP)}, pages 1:1--1:10, Eugene, OR, 2018.

\bibitem[You et~al.(2020)You, Li, Reddi, Hseu, Kumar, Bhojanapalli, Song,
  Demmel, Keutzer, and Hsieh]{you2019large}
Yang You, Jing Li, Sashank~J. Reddi, Jonathan Hseu, Sanjiv Kumar, Srinadh
  Bhojanapalli, Xiaodan Song, James Demmel, Kurt Keutzer, and Cho{-}Jui Hsieh.
\newblock Large batch optimization for deep learning: Training {BERT} in 76
  minutes.
\newblock In \emph{Proceedings of the 8th International Conference on Learning
  Representations (ICLR)}, Addis Ababa, Ethiopia, 2020.

\bibitem[Yu et~al.(2019)Yu, Jin, and Yang]{Proc:YuJY_ICML19}
Hao Yu, Rong Jin, and Sen Yang.
\newblock On the linear speedup analysis of communication efficient momentum
  {SGD} for distributed non-convex optimization.
\newblock In \emph{Proceedings of the 36th International Conference on Machine
  Learning (ICML)}, pages 7184--7193, Long Beach, CA, 2019.

\bibitem[Zeiler(2012)]{Z12}
Matthew~D Zeiler.
\newblock Adadelta: an adaptive learning rate method.
\newblock \emph{arXiv preprint arXiv:1212.5701}, 2012.

\bibitem[Zhou et~al.(2018{\natexlab{a}})Zhou, Chen, Cao, Tang, Yang, and
  Gu]{Arxiv:Zhou_18}
Dongruo Zhou, Jinghui Chen, Yuan Cao, Yiqi Tang, Ziyan Yang, and Quanquan Gu.
\newblock On the convergence of adaptive gradient methods for nonconvex
  optimization.
\newblock \emph{arXiv preprint arXiv:1808.05671}, 2018{\natexlab{a}}.

\bibitem[Zhou et~al.(2018{\natexlab{b}})Zhou, Chen, Cao, Tang, Yang, and
  Gu]{zhou2018convergence}
Dongruo Zhou, Jinghui Chen, Yuan Cao, Yiqi Tang, Ziyan Yang, and Quanquan Gu.
\newblock On the convergence of adaptive gradient methods for nonconvex
  optimization.
\newblock \emph{arXiv preprint arXiv:1808.05671}, 2018{\natexlab{b}}.

\bibitem[Zhou et~al.(2020)Zhou, Karimi, Yu, Xu, and Li]{zhou2020towards}
Yingxue Zhou, Belhal Karimi, Jinxing Yu, Zhiqiang Xu, and Ping Li.
\newblock Towards better generalization of adaptive gradient methods.
\newblock In \emph{Advances in Neural Information Processing Systems
  (NeurIPS)}, virtual, 2020.

\end{thebibliography}

\clearpage

\appendix

\section{Hyper-parameter Tuning and Algorithms} \label{app:experiment}

\subsection{The Adp-Fed Algorithm~\citep{reddi2020adaptive}}

The Adp-Fed (Adaptive Federated Optimization) is one of the baseline methods compared with Fed-LAMB in our paper. The algorithm is given in Algorithm~\ref{alg:adp-fed}. The key difference between Adp-Fed and Fed-AMS~\citep{chen2020toward} is that, in Adp-Fed, each client runs local SGD (Line~8), and an Adam optimizer is maintained for the global adaptive optimization (Line~15). In the Fed-AMS framework (as well as our Fed-LAMB), each clients runs local (adaptive) AMSGrad method, and the global model is simply obtained by averaging the local models.

\begin{algorithm}[H]
\caption{Adp-Fed: Adaptive Federated Optimization~\citep{reddi2020adaptive}} \label{alg:adp-fed}
\begin{algorithmic}[1]

\State \textbf{Input}: parameter $0< \beta_1, \beta_2 <1$, and learning rate $\alpha_t$, weight decaying parameter $\lambda \in [0,1]$.
\State \textbf{Initialize}: $\theta_{0,i} \in \Theta \subseteq \mathbb R^d $, $m_0=0$, $v_{0} =\epsilon$, $\forall i\in \llbracket n\rrbracket$, and $\theta_0 =  \frac{1}{n} \sum_{i=1}^n \theta_{0,i}$.
\vspace{0.05in}
\State \textbf{for $r=1, \ldots, R$ do}
\State $\quad$\textbf{parallel for device $i$ do}:
\State $\qquad$Set $\theta_{r,i}^{0} = \theta_{r-1}$.

\State $\qquad$\textbf{for $t=1, \ldots, T$ do}
\State $\qquad\quad$Compute stochastic gradient $g^t_{r,i}$ at $\theta_{r,i}^{0}$.
\State $\qquad\quad$$\theta_{r,i}^t=\theta_{r,i}^{t-1}-\eta_l g_{r,i}^t$ \label{adpfed line:local SGD}
\State $\qquad$\textbf{end for}

\State $\qquad$Devices send $\triangle_{r,i}=\theta_{r,i}^T-\theta_{r,i}^0$ to server.

\State $\quad$\textbf{end for}

\State \quad Server computes $\bar{\triangle}_r = \frac{1}{n}\sum_{i=1}^n \triangle_{r,i}$

\State \quad $m_r = \beta_1 m_{r-1} + (1-\beta_1)\bar{\triangle}_r$

\State \quad $v_r = \beta_2 v_{r-1} + (1-\beta_2)\bar{\triangle}_r^2$

\State \quad $\theta_r = \theta_{r-1}+\eta_g\frac{m_r}{\sqrt{v_r}}$ \label{adpfed line:global adam}

\State \textbf{end for}
\State \textbf{Output}: Global model parameter $\theta_R$.
\end{algorithmic}
\end{algorithm}

\subsection{Hyper-parameter Tuning}

In our empirical study, we tune the learning rate of each algorithm carefully such that the best performance is achieved. The search grids in all our experiments are provided in Table~{\ref{tab:tuning}}.

\begin{table}[h]
\centering
\caption{Search grids of the learning rate.}\label{tab:tuning}

\begin{tabular}{c|c}
\toprule[1pt]
 & Learning rate range     \\ \hline
Fed-SGD                  & $[0.001,0.003,0.005,0.01,0.03,0.05,0.1,0.3,0.5]$                      \\\hline
Fed-AMS                  & $[0.0001,0.0003,0.0005,0.001,0.003,0.005,0.01,0.03,0.05,0.1]$ \\\hline
Fed-LAMB     & $[0.001,0.003,0.005,0.01,0.03,0.05,0.1,0.3,0.5]$                      \\\hline
\multirow{2}{*}{Adp-Fed} & Local $\eta_l$: $[0.0001,0.0003,0.0005,0.001,0.003,0.005,0.01,0.03,0.05,0.1,0.3,0.5]$      \\
    & Global $\eta_g$: $[0.0001,0.0003,0.0005,0.001,0.003,0.005,0.01,0.03,0.05,0.1]$ \\\hline
Mime                  & $[0.0001,0.0003,0.0005,0.001,0.003,0.005,0.01,0.03,0.05,0.1]$ \\\hline
Mime-LAMB     & $[0.001,0.003,0.005,0.01,0.03,0.05,0.1,0.3,0.5]$                      \\
\toprule[1pt]
\end{tabular}

\end{table}

\vspace{1in}

\section{Theoretical Analysis}\label{app:proofs}

\subsection{Intermediary Lemma}

We now develop the proof of the convergence rate of Fed-LAMB. We need a supporting Lemma~\ref{lemma:iterates} for this.

\vspace{0.05in}
\begin{lem}\label{lemma:iterates}
Consider $\{\overline{\theta_r}\}_{r>0}$, the sequence of parameters obtained running Algorithm~\ref{alg:ldams}. Then for $i \in \inter$:
\beq\notag
\| \overline{\theta_r} - \theta_{r,i} \|^2 \leq \alpha^2 M^2 \phi_M^2 \frac{(1-\beta_2)p}{\epsilon} \eqsp,
\eeq
where $\phi_M$ is defined in Assumption~\ref{ass:phi} and p is the total number of dimensions $p = \sum_{\ell = 1}^\tot p_\ell$.
\end{lem}

\begin{proof}
Assuming the simplest case when $T=1$, i.e., one local iteration, then by construction of Algorithm~\ref{alg:ldams}, we have for all $\ell \in \llbracket \tot \rrbracket$, $i \in \inter$ and $r >0$:
\beq\notag
 \theta^{\ell}_{r,i} =  \overline{\theta_r}^{\ell}  - \alpha \phi(\|\theta_{r,i}^{\ell,t-1}\|)\psi_{r,i}^{j} / \|\psi_{r,i}^{\ell}\|=  \overline{\theta_r}^{\ell}  - \alpha \phi(\|\theta_{r,i}^{\ell,t-1}\|)
 \frac{m^{t}_{r,i}}{\sqrt{v^{t}_{r}}} \frac{1}{\|\psi_{r,i}^{\ell}\|}
\eeq
leading to
\beq\notag
\begin{split}
\|\overline{\theta_r}   -  \theta_{r,i}\|^2  = \sum_{\ell=1}^\tot \pscal{\overline{\theta_r}^{\ell}   -  \theta^{\ell}_{r,i}}{\overline{\theta_r}^{\ell}   -  \theta^{\ell}_{r,i}} \leq \alpha^2 M^2 \phi_M^2 \frac{(1-\beta_2)p}{\epsilon} \eqsp,
\end{split}
\eeq
which concludes the proof.
\end{proof}

\subsection{Proof of Theorem~\ref{th:multiple update}} \label{app:proofmain}

\begin{Theorem*}
Suppose \textbf{Assumption~\ref{ass:smooth}-Assumption~\ref{ass:phi}} hold. Consider $\{\overline{\theta_r}\}_{r>0}$, the sequence of parameters obtained running Algorithm~\ref{alg:ldams} with a constant learning rate $\alpha$. Let the number of local epochs be $T \geq 1$ and $\lambda = 0$. Then, for any round $R > 0$, we have
\begin{align}
  \frac{1}{R}\sum_{r=1}^R  \EE\left[ \left\| \frac{\nabla f(\overline{\theta_r})}{\hat v_r^{1/4}}   \right \|^2 \right] &\leq    \sqrt{\frac{M^2 p}{n}}  \frac{ \triangle}{\tot \alpha R}+\frac{4\alpha \alpha^2 L M^2 (T-1)^2 \phi_M^2 (1-\beta_2)p}{\sqrt{\epsilon}} \\\notag
&+4\alpha \frac{M^2}{\sqrt{\epsilon}} +      \frac{\phi_M   \sigma^2}{R n} \sqrt{\frac{1 - \beta_2}{M^2 p}  } +4\alpha \left[\phi_M \frac{\tot \sigma^2}{\sqrt{n}}\right]     + 4\alpha \left[ \phi_M^2\sqrt{M^2+p\sigma^2} \right],\notag
\end{align}
where $\triangle=\EE[f(\bar{\theta}_1)]  - \min \limits_{\theta \in \Theta} f(\theta)$.
\end{Theorem*}

\begin{proof}
Using Assumption~\ref{ass:smooth}, we have
\begin{align}\notag
f(\bar{\vartheta}_{r+1}) &  \leq f(\bar{\vartheta}_r) + \pscal{\nabla f(\bar{\vartheta}_r)}{\bar{\vartheta}_{r+1} - \bar{\vartheta}_r} + \sum_{\ell =1}^L \frac{L_\ell}{2} \| \bar{\vartheta}^\ell_{r+1} - \bar{\vartheta}^\ell_r \|^2\\\notag
&  \leq f(\bar{\vartheta}_r) + \sum_{\ell=1}^\tot \sum_{j=1}^{p_\ell} \nabla_{\ell} f(\bar{\vartheta}_r)^j (\bar{\vartheta}^{\ell,j}_{r+1} - \bar{\vartheta}^{\ell,j}_r) + \sum_{\ell =1}^L \frac{L_\ell}{2} \| \bar{\vartheta}^\ell_{r+1} - \bar{\vartheta}^\ell_r \|^2  \eqsp.
\end{align}
Taking expectations on both sides leads to
\begin{align}\label{eq:main}
- \EE[  \pscal{\nabla f(\bar{\vartheta}_r)}{\bar{\vartheta}_{r+1} - \bar{\vartheta}_r}]  \leq  \EE[ f(\bar{\vartheta}_r) - f(\bar{\vartheta}_{r+1})] + \sum_{\ell =1}^L \frac{L_\ell}{2} \EE[  \| \bar{\vartheta}^\ell_{r+1} - \bar{\vartheta}^\ell_r \|^2] \eqsp.
\end{align}

Yet, we observe that, using the classical intermediate quantity used for proving convergence results of adaptive optimization methods, see for instance~\citep{reddi2019convergence}, we have
\beq\label{eq:defseq}
\bar{\vartheta}_r = \bar{\theta}_r +  \frac{\beta_1}{1-\beta_1}(\bar{\theta}_{r} - \bar{\theta}_{r-1}) \eqsp,
\eeq
where $\bar{\theta_r}$ denotes the average of the local models at round $r$.
Then for each layer $\ell$,
\begin{align}\label{eq:gap}
\bar{\vartheta}^\ell_{r+1} - \bar{\vartheta}^\ell_r  & = \frac{1}{1-\beta_1}(\bar{\theta}^\ell_{r+1} - \bar{\theta}^\ell_{r}) - \frac{\beta_1}{1-\beta_1}(\bar{\theta}^\ell_{r} - \bar{\theta}^\ell_{r-1}) \nonumber\\
& = \frac{\alpha_{r}}{1-\beta_1} \frac{1}{n} \sum_{i = 1}^n \frac{\phi(\|\theta_{r,i}^{\ell}\|)}{\|\psi_{r,i}^{\ell}\|} \psi_{r,i}^{\ell}  - \frac{\alpha_{r-1}}{1-\beta_1} \frac{1}{n} \sum_{i = 1}^n \frac{\phi(\|\theta_{r-1,i}^{\ell}\|)}{\|\psi_{r-1,i}^{\ell}\|} \psi_{r-1,i}^{\ell}\nonumber\\
& = \frac{\alpha \beta_1}{1-\beta_1} \frac{1}{n}  \sum_{i = 1}^n  \left( \frac{\phi(\|\theta_{r,i}^{\ell}\|)}{\sqrt{v^{t}_{r}} \|\psi_{r,i}^{\ell}\|} - \frac{\phi(\|\theta_{r-1,i}^{\ell}\|)}{\sqrt{v^{t}_{r-1}} \|\psi_{r-1,i}^{\ell}\|} \right) m^{t}_{r-1} + \frac{\alpha}{n} \sum_{i = 1}^n \frac{\phi(\|\theta_{r,i}^{\ell}\|)}{\sqrt{v^{t}_{r}} \|\psi_{r,i}^{\ell}\|} g^t_{r,i} \eqsp,
\end{align}
where we have assumed a constant learning rate $\alpha$.

We note for all $\theta \in \Theta$, the majorant $G > 0$ such that $\phi(\|\theta \|) \leq G$.
Then, following \eqref{eq:main}, we obtain
\begin{align}\label{eq:main2}
- \EE[  \pscal{\nabla f(\bar{\vartheta}_r)}{\bar{\vartheta}_{r+1} - \bar{\vartheta}_r}]  \leq  \EE[ f(\bar{\vartheta}_r) - f(\bar{\vartheta}_{r+1})] + \sum_{\ell =1}^L \frac{L_\ell}{2} \EE[  \| \bar{\vartheta}_{r+1} - \bar{\vartheta}_r \|^2] \eqsp.
\end{align}
Developing the LHS of \eqref{eq:main2} using \eqref{eq:gap} leads to
\begin{align} \notag
\pscal{\nabla f(\bar{\vartheta}_r)}{\bar{\vartheta}_{r+1} - \bar{\vartheta}_r} &= \sum_{\ell=1}^\tot \sum_{j=1}^{p_\ell} \nabla_{\ell} f(\bar{\vartheta}_r)^j (\bar{\vartheta}^{\ell,j}_{r+1} - \bar{\vartheta}^{\ell,j}_r)  \\ \notag
& =  \frac{\alpha \beta_1}{1-\beta_1}\frac{1}{n}  \sum_{\ell=1}^\tot \sum_{j=1}^{p_\ell} \nabla_{\ell} f(\bar{\vartheta}_r)^j \left[   \sum_{i = 1}^n  \left( \frac{\phi(\|\theta_{r,i}^{\ell}\|)}{\sqrt{v^{t}_{r}} \|\psi_{r,i}^{\ell}\|} - \frac{\phi(\|\theta_{r-1,i}^{\ell}\|)}{\sqrt{v^{t}_{r-1}} \|\psi_{r-1,i}^{\ell}\|} \right) m^{t}_{r-1}  \right] \\ \label{eqn1}
& \underbrace{ -\frac{\alpha}{n} \sum_{\ell=1}^\tot \sum_{j=1}^{p_\ell} \nabla_{\ell} f(\bar{\vartheta}_r)^j  \sum_{i = 1}^n \frac{\phi(\|\theta_{r,i}^{\ell}\|)}{\sqrt{v^{t}_{r}} \|\psi_{r,i}^{\ell}\|} g_{r,i}^{t,l,j}}_{= A_1}   \eqsp.
\end{align}
Suppose $T$ is the total number of local iterations and $R$ is the number of rounds. We can write~\eqref{eqn1}~as
\begin{align}\notag
    A_1=-\alpha \langle \nabla f(\bar \vartheta_r),\frac{\bar g_r}{\sqrt{\hat v_r}} \rangle,
\end{align}
where $\bar g_r=\frac{1}{n}\sum_{i=1}^n \bar g_{t,i}$, with $\bar g_{t,i}=\Big[\frac{\phi(\Vert \theta_{t,i}^1\Vert)}{\Vert \psi_{t,i}^1\Vert}g_{t,i}^1,..., \frac{\phi(\Vert \theta_{t,i}^L\Vert)}{\Vert \psi_{t,i}^L\Vert}g_{t,i}^L   \Big]$ representing the normalized gradient (concatenated by layers) of the $i$-th device. It holds that
\begin{align}
    \langle \nabla f(\bar \vartheta_r),\frac{\bar g_r}{\sqrt{\hat v_r}} \rangle&=\frac{1}{2}\Vert \frac{\nabla f(\bar\vartheta_r) }{\hat v_r^{1/4}}\Vert^2+\frac{1}{2}\Vert \frac{\bar g_r }{\hat v_r^{1/4}}\Vert^2-\Vert \frac{\nabla f(\bar\vartheta_r)-\bar g_r }{\hat v_r^{1/4}}\Vert^2.  \label{eqn:x1}
\end{align}

To bound the last term on the RHS, we have
\begin{align}\notag
    \Vert \frac{\nabla f(\bar\vartheta_r)-\bar g_r }{\hat v_r^{1/4}}\Vert^2=\Vert \frac{\frac{1}{n}\sum_{i=1}^n (\nabla f(\bar\vartheta_r)-\bar g_{t,i})}{\hat v_r^{1/4}} \Vert^2
    &\leq \frac{1}{n}\sum_{i=1}^n\Vert \frac{\nabla f(\bar\vartheta_r)-\bar g_{t,i}}{\hat v_r^{1/4}} \Vert^2\\\notag
    &\leq \frac{2}{n}\sum_{i=1}^n \Big(\Vert \frac{\nabla f(\bar\vartheta_r)-\nabla f(\bar\theta_r)}{\hat v_r^{1/4}} \Vert^2+\Vert \frac{\nabla f(\bar\theta_r)-\bar g_{t,i}}{\hat v_r^{1/4}} \Vert^2  \Big).
\end{align}
By Lipschitz smoothness of the loss function, the first term admits
\begin{align}\notag
    \frac{2}{n}\sum_{i=1}^n\Vert \frac{\nabla f_i(\bar\vartheta_r)-\nabla f_i(\bar\theta_r)}{\hat v_r^{1/4}} \Vert^2 \leq \frac{2}{n \sqrt{\epsilon}}\sum_{i=1}^n L_\ell\Vert \bar\vartheta_r-\bar\theta_r\Vert^2 & =\frac{2L_\ell}{n \sqrt{\epsilon}}\frac{\beta_1^2}{(1-\beta_1)^2}\sum_{i=1}^n \Vert \bar\theta_r-\bar\theta_{t-1}\Vert ^2\\\notag
    &\leq \frac{2\alpha^2 L_\ell }{n \sqrt{\epsilon}}\frac{\beta_1^2}{(1-\beta_1)^2} \sum_{l=1}^L \sum_{i=1}^n\Vert \frac{\phi(\Vert \theta_{t,i}^l\Vert)}{\Vert \psi_{t,i}^l\Vert}\psi_{t,i}^l \Vert^2\\\notag
    &\leq \frac{2\alpha^2 L_\ell p\phi_M^2}{ \sqrt{\epsilon}}\frac{\beta_1^2}{(1-\beta_1)^2}.
\end{align}
For the second term,
\begin{align}\label{eq:inter}
    \frac{2}{n}\sum_{i=1}^n\Vert \frac{\nabla f(\bar\theta_r)-\bar g_{t,i}}{\hat v_r^{1/4}} \Vert^2 \leq \frac{4}{n}\Big( \underbrace{\sum_{i=1}^n \Vert \frac{\nabla f(\bar\theta_r)-\nabla f(\theta_{t,i})}{\hat v_r^{1/4}} \Vert^2}_{B_1} + \underbrace{ \sum_{i=1}^n\Vert \frac{\nabla f(\theta_{t,i})-\bar g_{t,i}}{\hat v_r^{1/4}} \Vert^2}_{B_2} \Big).
\end{align}
Using the smoothness of $f_i$ we can transform $B_1$ into consensus error by
\begin{align}\notag
    B_1\leq \frac{L}{\sqrt{\epsilon}}\sum_{i=1}^n \Vert \bar\theta_r - \theta_{t,i}\Vert^2  & =\frac{\alpha^2 L}{\sqrt{\epsilon}}\sum_{i=1}^n\sum_{l=1}^L \| \sum_{j=\lfloor t \rfloor_r+1}^t \Big( \frac{\phi(\Vert \theta_{j,i}^l\Vert)}{\Vert \psi_{j,i}^l\Vert}\psi_{j,i}^l-\frac{1}{n}\sum_{k=1}^n \frac{\phi(\Vert \theta_{j,k}^l\Vert)}{\Vert \psi_{j,k}^l\Vert}\psi_{j,k}^l \Big) \|^2\\\label{eqn:B1}
    &\leq n \frac{\alpha^2 L}{\sqrt{\epsilon}} M^2 (T-1)^2 \phi_M^2 (1-\beta_2)p,
\end{align}
where the last inequality stems from Lemma~\ref{lemma:iterates} in the particular case where $  \theta_{t,i}$ are averaged every $ct+1$ local iterations for any integer $c$, since $(t-1)-(\lfloor t \rfloor_r+1)+1 \leq T-1$.

We now develop the expectation of $B_2$ under the simplification that $\beta_1 = 0$:
\begin{align}\notag
    \mathbb E[B_2]&=\mathbb E[\sum_{i=1}^n\Vert \frac{\nabla f(\theta_{t,i})-\bar g_{t,i}}{\hat v_r^{1/4}} \Vert^2] \\\notag
    &\leq \frac{nM^2}{\sqrt{\epsilon}}+n\phi_M^2\sqrt{M^2+p\sigma^2}-2\sum_{i=1}^n\mathbb E[\langle \nabla f(\theta_{t,i}),\bar g_{t,i} \rangle/\sqrt{\hat v_r}]\\\notag
    &=\frac{nM^2}{\sqrt{\epsilon}}+n\phi_M^2\sqrt{M^2+p\sigma^2}-2\sum_{i=1}^n \sum_{\ell=1}^L \mathbb E[\langle \nabla_\ell f(\theta_{t,i}),\frac{\phi(\|\theta_{t,i}^l \|)}{\| \psi_{t,i}^l \|}g_{t,i}^l \rangle/\sqrt{\hat v_r^l}]\\\notag
    &=\frac{nM^2}{\sqrt{\epsilon}}+n\phi_M^2\sqrt{M^2+p\sigma^2}-2\sum_{i=1}^n \sum_{l=1}^L\sum_{i=1}^{p_l} \mathbb E[\nabla_l f(\theta_{t,i})^j\frac{\phi(\|\theta_{t,i}^{l,j} \|)}{\sqrt{\hat v_r^{l,j}}\| \psi_{t,i}^{l,j} \|}g_{t,i}^{l,j} ]\\\notag
    & \leq \frac{nM^2}{\sqrt{\epsilon}}+n\phi_M^2\sqrt{M^2+p\sigma^2}-2\sum_{i=1}^n \sum_{l=1}^L\sum_{i=1}^{p_l} \mathbb E \left[ \sqrt{\frac{1-\beta_2}{M^2 p_\ell}}  \phi(\|\theta_{r,i}^{l,j}\|)  \nabla_l f(\theta_{t,i})^j  g_{t,i}^{l,j}\right]\\\notag
    &\hspace{0.4in} -2 \sum_{i = 1}^n \sum_{l=1}^L\sum_{j=1}^{p_l}  E \left[  \left( \phi(\|\theta_{r,i}^{l,j}\|)   \nabla_l f(\theta_{t,i})^j   \frac{g_{r,i}^{t,l,j}}{ \|\psi_{r,i}^{l,j}\|}\right)\mathsf{1}\left( \sign(  \nabla_l f(\theta_{t,i})^j \neq  \sign( g_{r,i}^{t,l,j}) \right)\right],
\end{align}
where we use assumption Assumption~\ref{ass:boundgrad}, Assumption~\ref{ass:var} and Assumption~\ref{ass:phi}.
Yet,
\begin{align*}
&- \mathbb E \Bigg[  \left( \phi(\|\theta_{r,i}^{l,j}\|)   \nabla_l f(\theta_{t,i})^j   \frac{g_{r,i}^{t,l,j}}{ \|\psi_{r,i}^{l,j}\|}\right)\mathsf{1}\left( \sign(  \nabla_l f(\theta_{t,i})^j
\neq  \sign( g_{r,i}^{t,l,j}) \right)\Bigg] \\
&\hspace{2in} \leq  \phi_M \nabla_l f(\theta_{t,i})^j   \mathbb{P}\left[  \sign(  \nabla_l f(\theta_{t,i})^j \neq  \sign( g_{r,i}^{t,l,j}) \right].
\end{align*}
Then we have
\begin{align}\notag
    \mathbb E[B_2]\leq  \frac{nM^2}{\sqrt{\epsilon}}+n\phi_M^2\sqrt{M^2+p\sigma^2}-2 \phi_m \sqrt{\frac{1-\beta_2}{M^2 p}} \sum_{i=1}^n \E[\| [\nabla f(\theta_{t,i}) \|^2] + \phi_M \frac{\tot \sigma^2}{\sqrt{n}}
\end{align}
Thus, \eqref{eq:inter} becomes
\begin{align}\notag
    \frac{2}{n}\sum_{i=1}^n\Vert \frac{\nabla f_i(\bar\theta_r)-\bar g_{t,i}}{\hat v_r^{1/4}} \Vert^2 \leq 4 \left[ \frac{\alpha^2 L_\ell}{\sqrt{\epsilon}} \alpha^2 M^2 (T-1)^2 \phi_M^2 (1-\beta_2)p + \frac{\alpha M^2}{\sqrt{\epsilon}}+\phi_M^2\sqrt{M^2+p\sigma^2} + \alpha\phi_M \frac{\tot \sigma^2}{\sqrt{n}}\right]
\end{align}

Substituting all ingredients into (\ref{eqn:x1}), we obtain
\begin{align}\notag
    -\alpha \mathbb E[\langle \nabla f(\bar \vartheta_r),\frac{\bar g_r}{\sqrt{\hat v_r}} \rangle] &\leq -\frac{\alpha}{2}\mathbb E\big[\Vert \frac{\nabla f(\bar\vartheta_r) }{\hat v_r^{1/4}}\Vert^2 \big]-\frac{\alpha}{2}\mathbb E\big[\Vert \frac{\bar g_r }{\hat v_r^{1/4}}\Vert^2 \big]+\frac{2\alpha^3 L_\ell p\phi_M^2}{ \sqrt{\epsilon}}\frac{\beta_1^2}{(1-\beta_1)^2} \\\notag
    &\hspace{0.1in}  + 4 \alpha \left[ \frac{\alpha^2 L}{\sqrt{\epsilon}} M^2 (T-1)^2 \phi_M^2 (1-\beta_2)p + \frac{\alpha M^2}{\sqrt{\epsilon}}+\phi_M^2\sqrt{M^2+p\sigma^2} +\alpha \phi_M \frac{\tot \sigma^2}{\sqrt{n}}\right].
\end{align}

To bound the second term on the RHS in above, we notice that
\begin{align}\notag
    \mathbb E\big[\Vert \frac{\bar g_r }{\hat v_r^{1/4}}\Vert^2\big]=\frac{1}{n^2}\mathbb E\big[\Vert \frac{\sum_{i=1}^n \bar g_{r,i}}{\hat v_r^{1/4}}\Vert^2 \big] &=\frac{1}{n^2}\mathbb E\big[ \sum_{l=1}^L \sum_{i=1}^n \Vert  \frac{\phi(\Vert \theta_{r,i}^l\Vert)}{\hat v^{1/4} \Vert \psi_{r,i}^l\Vert}g_{r,i}^l \Vert^2 \big] \\
    &\geq \phi_m^2(1-\beta_2) \mathbb E\left[ \Vert \frac{1}{n}\sum_{i=1}^n \frac{\nabla f(\theta_{r,i})}{\hat v^{1/4}_r} \Vert^2 \right]\\\notag
    &=\phi_m^2(1-\beta_2) \mathbb E\left[ \Vert  \frac{\overline\nabla f(\theta_{r})}{\hat v^{1/4}_r} \Vert^2 \right]. \notag
\end{align}

Regarding $\left\| \frac{\overline{\nabla}f(\theta_r)}{\hat v_r^{1/4}} \right\|^2$, we have
\begin{align*}
\left\| \frac{\overline{\nabla}f(\theta_r)}{\hat v_r^{1/4}} \right\|^2 & \geq \frac{1}{2} \left\| \frac{\nabla f(\overline{\theta_r})}{\hat v_r^{1/4}} \right\|^2 - \left\| \frac{\overline{\nabla}f(\theta_r)- \nabla f(\overline{\theta_r})}{\hat v_r^{1/4}} \right\|^2\\
& \geq \frac{1}{2} \left\| \frac{\nabla f(\overline{\theta_r})}{\hat v_r^{1/4}} \right\|^2 - \left\| \frac{\frac{1}{n}\sum_{i=1}^n (\nabla f_i(\theta_{r})-\nabla f(\bar\theta_r))}{\hat v_r^{1/4}} \right\|^2 \\
&\geq \frac{1}{2} \left\| \frac{\nabla f(\overline{\theta_r})}{\hat v_r^{1/4}} \right\|^2 - \frac{\alpha^2 L_\ell}{\sqrt{\epsilon}} M^2 (T-1)^2 (\sigma^2 + G^2) (1-\beta_2)p,
\end{align*}
where the last line is due to (\ref{eqn:B1}) and Assumption~\ref{ass:var}. Therefore, we have obtained
\begin{align*}
    A_1&\leq -\frac{\alpha\phi_m^2(1-\beta_2)}{4}\left\| \frac{\nabla f(\overline{\theta_r})}{\hat v_r^{1/4}} \right\|^2+\frac{\alpha^3 L_\ell}{\sqrt{\epsilon}} M^2 (T-1)^2 \phi_m^2\phi_M^2 (1-\beta_2)^2p+\frac{2\alpha^3 L_\ell p\phi_M^2}{ \sqrt{\epsilon}}\frac{\beta_1^2}{(1-\beta_1)^2} \\
    &\hspace{0.4in}  + 4\alpha \left[ \frac{\alpha^2 L}{\sqrt{\epsilon}}  M^2 (T-1)^2 (\sigma^2 + G^2) (1-\beta_2)p + \frac{M^2 \alpha }{\sqrt{\epsilon}}+ \alpha \phi_M^2\sqrt{M^2+p\sigma^2} + \phi_M \alpha \frac{\tot \sigma^2}{\sqrt{n}}\right],\\
    &\leq -\frac{\alpha\phi_m^2(1-\beta_2)}{4}\left\| \frac{\nabla f(\overline{\theta_r})}{\hat v_r^{1/4}} \right\|^2+\frac{\alpha^3 L_\ell}{\sqrt{\epsilon}} M^2 (T-1)^2 \phi_m^2\phi_M^2 (1-\beta_2)^2p+\frac{2\alpha^3 L_\ell p\phi_M^2}{ \sqrt{\epsilon}}\frac{\beta_1^2}{(1-\beta_1)^2} \\
    &\hspace{0.4in}  + 4 \alpha\Big[ \frac{\alpha^2 L}{\sqrt{\epsilon}}  M^2 (T-1)^2 G^2 (1-\beta_2)p + \frac{M^2 \alpha }{\sqrt{\epsilon}}+ \alpha \phi_M^2\sqrt{M^2+p\sigma^2} \\
    &\hspace{1.7in} + \sigma^2 \left(\frac{\alpha^2 L}{\sqrt{\epsilon}}  M^2 (T-1)^2(1-\beta_2)p+ \phi_M \alpha \frac{\tot }{\sqrt{n}} \right)\Big].
\end{align*}
Substitute back into (\ref{eqn1}), assuming $M\leq 1$, we have the following by taking the telescope sum
\begin{align*}
    &\frac{1}{R}\sum_{t=1}^R  \EE\left[ \left\| \frac{\nabla f(\overline{\theta_r})}{\hat v_r^{1/4}}   \right \|^2 \right] \\
    & \lesssim  \sqrt{\frac{M^2 p}{n}} \frac{ f(\bar{\vartheta}_1)  - \EE[ f(\bar{\vartheta}_{R+1})]}{\tot \alpha R}+   \frac{\alpha}{n^2}  \sum_{r=1}^R  \sum_{i = 1}^n  \sigma_i^2 \EE\left[ \left\|\frac{\phi(\|\theta_{r,i}^{\ell}\|)}{\sqrt{v_r} \|\psi_{r,i}^{\ell}\|} \right\|^2 \right] +\frac{2\alpha^3 \overline{L} p\phi_M^2}{ \sqrt{\epsilon}}\frac{\beta_1^2}{(1-\beta_1)^2} \nonumber\\
   &   +4 \Big[ \frac{\alpha^2 \overline{L}}{\sqrt{\epsilon}}  M^2 (T-1)^2 G^2 (1-\beta_2)p + \frac{\alpha M^2}{\sqrt{\epsilon}}+ \alpha \phi_M^2 \sqrt{M^2+p\sigma^2} \nonumber \\
   & + \sigma^2 \left(\frac{\alpha^2 \overline{L}}{\sqrt{\epsilon}}  M^2 (T-1)^2(1-\beta_2)p+ \phi_M\alpha \frac{\tot }{\sqrt{n}} \right)\Big]  +\frac{\alpha \beta_1}{1-\beta_1}  \sqrt{(1-\beta_2)p} \frac{\tot M^2}{\sqrt{\epsilon}} +\overline{L} \alpha^2 M^2 \phi_M^2 \frac{(1-\beta_2)p}{T\epsilon} \nonumber\\
   & \leq   \sqrt{\frac{M^2 p}{n}}  \frac{ \EE[f(\bar{\theta}_1)]  - \min \limits_{\theta \in \Theta} f(\theta)}{\tot \alpha R} +      \frac{\phi_M   \sigma^2}{R n} \sqrt{\frac{1 - \beta_2}{M^2 p}  } \\
   &+4 \Big[ \frac{\alpha^2 \overline{L}}{\sqrt{\epsilon}}  M^2 (T-1)^2 G^2 (1-\beta_2)p + \frac{M^2 \alpha }{\sqrt{\epsilon}}+\phi_M^2 \alpha \sqrt{M^2+p\sigma^2} \\
   & + \sigma^2 \Big(\frac{\alpha^2 \overline{L}}{\sqrt{\epsilon}}  M^2 (T-1)^2(1-\beta_2)p+ \phi_M \frac{\tot }{\sqrt{n}} \Big)\Big]+\frac{\alpha \beta_1}{1-\beta_1}  \sqrt{(1-\beta_2)p} \frac{\tot M^2}{\sqrt{\epsilon}} \\
   &\hspace{1.4in} +\overline{L} \alpha^2 M^2 \phi_M^2 \frac{(1-\beta_2)p}{T\epsilon} +\frac{2\alpha^3 \overline{L} p\phi_M^2}{ \sqrt{\epsilon}}\frac{\beta_1^2}{(1-\beta_1)^2}.
\end{align*}

Organizing terms, we conclude the proof.
\end{proof}

\end{document}